\documentclass[letterpaper]{article} % DO NOT CHANGE THIS
\usepackage{aaai2027}
\usepackage[hyphens]{url}  % DO NOT CHANGE THIS
\usepackage{graphicx} % DO NOT CHANGE THIS
\graphicspath{{aaai2027/figures/}{figures/}}
\usepackage{natbib}  % DO NOT CHANGE THIS AND DO NOT ADD ANY OPTIONS TO IT
\usepackage{caption} % DO NOT CHANGE THIS AND DO NOT ADD ANY OPTIONS TO IT
\usepackage{algorithm}
\usepackage{algorithmic}
\usepackage{amsmath}
\usepackage{amssymb}
\usepackage{amsthm}
\usepackage{xcolor}

\newtheorem{proposition}{Proposition}
\newtheorem{theorem}{Theorem}
\newtheorem{lemma}{Lemma}

\newif\ifshowreviewnotes
\showreviewnotesfalse
\ifshowreviewnotes
  \newcommand{\todo}[1]{\textcolor{red}{[todo: #1]}}
  \newcommand{\addressed}[1]{\textcolor{blue}{[addressed: #1]}}
\else
  \newcommand{\todo}[1]{}
  \newcommand{\addressed}[1]{}
\fi

\usepackage{booktabs}
\usepackage{multirow}

\title{StochSIPP: Safe Interval Path Planning in Stochastic Dynamic Environments}
\author{
    Ajith Kemisetti\textsuperscript{\rm 1},
    Shahaf S. Shperberg\textsuperscript{\rm 2},
    Yoonchang Sung\textsuperscript{\rm 3}\corresponding
}
\affiliations{
    \textsuperscript{\rm 1}The University of Texas at Austin\\
    \textsuperscript{\rm 2}Ben Gurion University of the Negev\\
    \textsuperscript{\rm 3}Nanyang Technological University\\
}

\begin{document}

\maketitle

\begin{abstract}
Safe navigation under uncertain time-dependent blockage requires anticipating observations before committing to motion. We present StochSIPP, an exact contingent planner for temporal roadmaps with uncertain edge and vertex statuses revealed locally during execution. StochSIPP uses SIPP to generate certified-safe macro-actions that terminate at the next observation or the goal, and bounded AND/OR search over a cached action--observation graph to select actions for every reachable observation outcome. Optimistic and robust SIPP relaxations provide admissible lower and upper bounds for bounded AND/OR search. When every interval declared deterministically safe is truly safe, sensing is exact, and execution follows the planned timing, the resulting policy is provably collision-free. With correct independent probabilities and complete action and outcome generation, it minimizes expected arrival time within the roadmap and horizon. Experiments on controlled roadmap instances show that StochSIPP preserves the observed success of safe fixed-path baselines while reducing arrival time, and solves gated scenarios in which conservative fixed-path planners return no plan. A scalability study further reveals rapid growth as the number of simultaneously observed uncertain statuses increases.
\end{abstract}


%%%%%%%%%%%%%%%%%
\section{Introduction}
\label{sec:intro}

Safe navigation among moving obstacles is a fundamental requirement for autonomous robots operating in human environments, warehouses, and other dynamic spaces. This problem becomes particularly challenging when obstacle motion is uncertain. A robot may have prior knowledge of likely motion patterns, but the trajectory that will actually occur may not be known until execution. Planning only for a nominal prediction can therefore lead to collisions, while treating all possible obstacle motions as simultaneously present can be unnecessarily conservative.

Existing approaches plan from a nominal prediction, conservatively account for
several possible motions~\cite{castillo2020real,de2021scenario,liu2023radius}, or
react after behavior is
observed~\cite{koenig2002d,likhachev2005anytime,park2012itomp}. We study a
distinct input model: a predictor supplies fixed safety
probabilities for specified roadmap edge and vertex uses over time. Each
uncertain interval has an episode-fixed safe or blocked status. The robot
senses locally relevant statuses before motion; observing one status does not
update unrelated probabilities. The planner therefore does not infer a
complete obstacle-trajectory realization.

The objective is a contingent policy, rather than a single path under an
expected obstacle trajectory. The policy anticipates future local
observations and specifies which certified-safe motion to execute after each
observation. Figure~\ref{fig:gazebo} illustrates this setting with a warehouse example, where the robot must decide whether to approach an aisle whose availability is not yet known.

\begin{figure}[t]
\centering
\includegraphics[width=0.9\columnwidth]{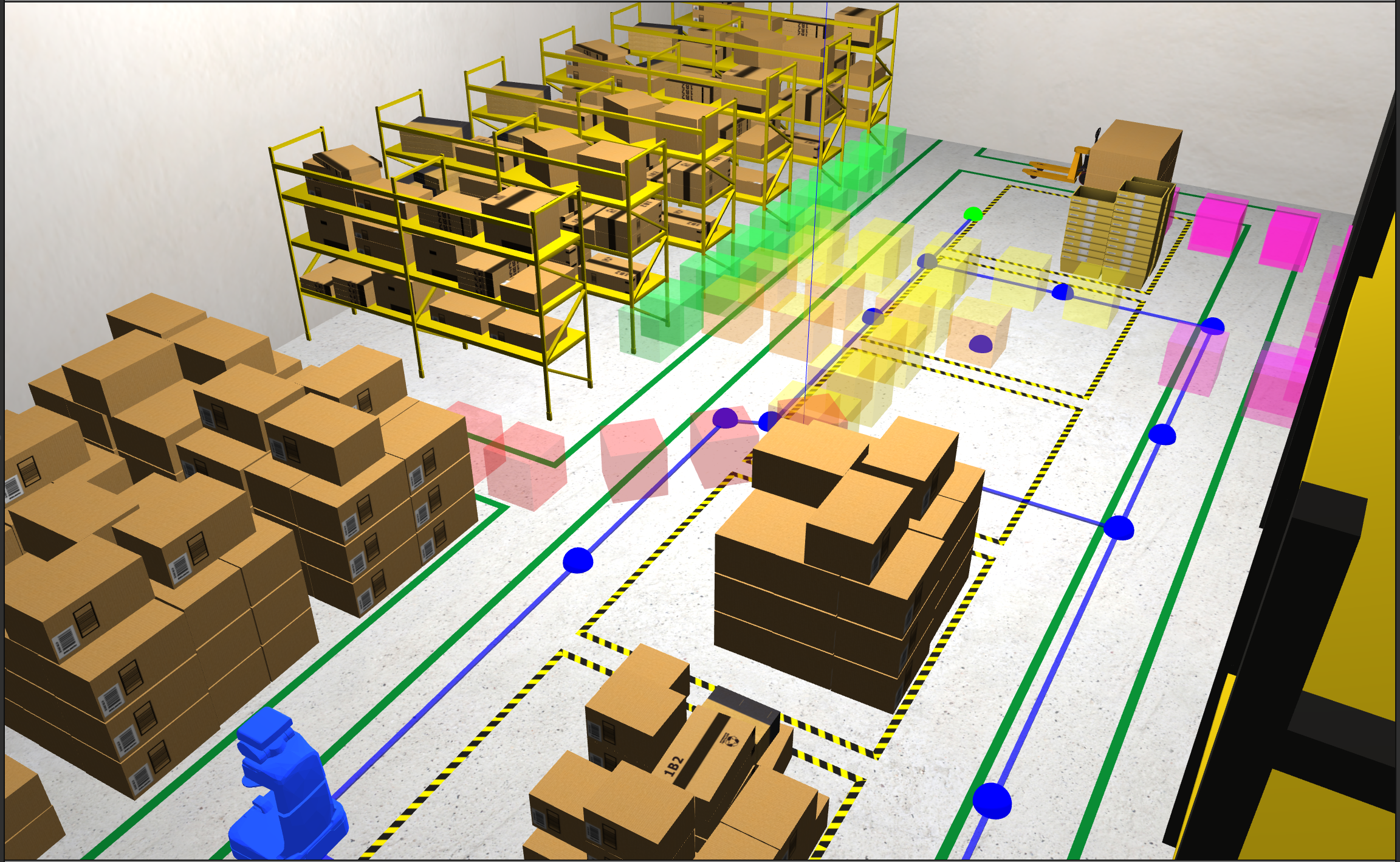}
\caption{A robot executing a StochSIPP policy in the Gazebo warehouse. The blue dots and lines correspond to roadmap vertices and edges, respectively, and the green dot marks the goal. The overlaid trajectories illustrate possible realizations of the moving obstacles.}
\label{fig:gazebo}
\end{figure}

Our approach builds on Safe Interval Path Planning
(SIPP~\cite{Phillips2011SIPPSI}), which plans around known dynamic-obstacle
trajectories by representing collision-free time intervals. We associate
roadmap edge and vertex intervals with supplied safety probabilities, thereby
representing uncertain temporal availability without expanding every
timestep in the search graph.

The model generalizes the Canadian Traveler Problem
(CTP~\cite{bar1991canadian}): time-indexed uncertain edge and vertex uses
replace CTP's uncertain roads, and local observations reveal their statuses. Unlike static road uncertainty, temporal uncertainty couples route choice, arrival time, waiting, and future information: reaching the same vertex at different times can produce different feasible motions and observation opportunities. OR nodes
represent decisions, AND nodes represent observation outcomes, and a solution
subgraph selects a certified-safe motion for every reachable outcome.

This paper makes three contributions. First, it defines a finite model of independent interval-status variables and a lossless SIPP macro-action abstraction for exact contingent search without explicit time expansion. Second, it derives admissible SIPP lower and upper bounds, sound
pruning rules, collision-free execution conditions, and finite-roadmap
optimality; the decision problem is PSPACE-hard. Third, it evaluates the
value of observing one uncertain gate before route selection and measures the
growth caused by explicit simultaneous-outcome branching.

On fifty controlled maps, StochSIPP matched the observed success of every safe
fixed-path baseline while reducing its mean arrival time, whereas the
risk-taking baselines recorded shorter durations but failed in about half of the
trials. When every route to the goal contained an uncertain gate, the conservative baselines returned no plan at all.

%%%%%%%%%%%%%%%%%
\section{Related Work and Background}
\label{sec:related}

This section positions StochSIPP relative to dynamic-obstacle planning, SIPP,
and contingent graph search. Reactive controllers and incremental planners use
current observations or revise an existing
path~\cite{fiorini1998motion,yang2019sampling,majumdar2017funnel,otte2016rrtx,ojha2025dynamic},
while prediction-based methods plan from nominal trajectories, chance
constraints, or updated
predictions~\cite{kuderer2012feature,castillo2020real,park2012itomp}. StochSIPP
addresses the complementary setting in which fixed temporal safety probabilities
and local visibility are supplied before planning.

SIPP represents known dynamic-obstacle trajectories through maximal
collision-free intervals~\cite{Phillips2011SIPPSI}. Its extensions cover
anytime, multi-agent, kinodynamic, and manipulator
planning~\cite{sippvariant,sippmapf,ali2023safe,kerimov2025safe}.
Any-start-time SIPP
precomputes arrival-time functions that can answer cost-to-go queries for
different departure times~\cite{thomas2023anystarttime}; StochSIPP uses
deterministic relaxations of these functions as bounds.

CTP asks for a shortest-path policy when uncertain edge statuses are revealed
locally~\cite{bar1991canadian}. Its AND--OR representation alternates choices
with observations, and admissible AND--OR search solves it
exactly~\cite{aostar_paper,cao_paper}. Stochastic CTP with independent edge
statuses is PSPACE-complete~\cite{fried2013complexity}. StochSIPP extends this
structure with time-indexed edge and vertex statuses, positive motion
durations, waiting, and a finite horizon. These temporal features invalidate ordinary earliest-arrival dominance and require preserving endpoint times at which different observations and subsequent actions become available. Rather than explicitly time-expanding this temporal problem into an ordinary CTP instance, StochSIPP uses SIPP to generate lossless certified-safe macro-actions and reusable cost bounds within the contingent search.

%%%%%%%%%%%%%%%%%
\section{Problem Formulation}
\label{sec:prob}

This section defines the finite stochastic model solved by StochSIPP. The
input gives a safety probability for every relevant roadmap use and time.
Sensing stores an observed status in $y$ and leaves the parameters of all
unobserved variables unchanged; the model never maintains a distribution over
complete obstacle trajectories.

\subsection{Finite Temporal Roadmap}

The finite directed roadmap is
$G=(V,E,W)$, where the vertices $V$ are robot configurations that avoid
all static obstacles, $E\subseteq V\times V$ contains statically feasible
motion primitives, and $W:E\rightarrow\mathbb{N}_{>0}$ gives their integer
durations. Every $q\in V$ has a unit-duration wait self-loop
$w_q=(q,q)\in E$, with $W(w_q)=1$. The start and goal satisfy
$q_s,q_g\in V$. Time belongs to the finite set
\[
    \mathcal{T}=\{0,1,\ldots,T_{\max}\},
\]
where $T_{\max}\in\mathbb{N}_{>0}$ is the planning horizon. An edge $e$ may
depart at time $t$ only if $t+W(e)\leq T_{\max}$. Edges with
$W(e)>T_{\max}$ are omitted, so every edge has at least one valid departure
time.

A vertex safety event at time $t$ states whether the robot may occupy $q$ at
$t$; an edge safety event at departure time $t$ states whether the complete
execution of $e$ over $[t,t+W(e))$ is collision-free. The event for a wait
self-loop therefore covers the complete unit-time occupancy, not only its
endpoints.

\subsection{Independent Bernoulli Safety Variables}

For a roadmap object $z\in V\cup E$, let $\mathcal{D}_z$ be its set of valid
times: $\mathcal{D}_q=\mathcal{T}$ for a vertex, and
$\mathcal{D}_e=\{t\in\mathcal{T}:t+W(e)\leq T_{\max}\}$ for an edge.
The input partitions $\mathcal{D}_z$ into consecutive integer intervals
$I_z^1,\ldots,I_z^{K_z}$. For $t\in\mathcal{D}_z$, let
$\kappa(z,t)$ be the unique index $k$ for which $t\in I_z^k$. Each interval
has one binary safety variable $X_z^k\sim\operatorname{Bernoulli}(p_z^k)$,
where $p_z^k\in[0,1]$ is its safety probability. All interval variables are
mutually independent. Their values are sampled once at the start of an
episode and remain fixed during that episode. A use of object $z$ at any
$t\in I_z^k$ is safe exactly when $X_z^k=1$.

Per-timestep marginals can be entered as singleton intervals
$I_z^k=\{t\}$, but the model then assumes their statuses are independent. A
longer interval is valid only when its times share one Bernoulli status; equal
probabilities at consecutive times do not imply a shared realized status and so
do not justify merging those times. If the supplied model contains only edge
uncertainty, all vertex variables are deterministically one.

Let
$
    \mathcal{U}
    =
    \{(z,k):z\in V\cup E,\ 1\leq k\leq K_z,\
      0<p_z^k<1\}
$
be the set of uncertain variables, and let $M=|\mathcal{U}|$. Variables
with $p_z^k=1$ are declared deterministically safe, and variables with $p_z^k=0$ are declared deterministically blocked. We assume initial occupancy is safe:
$p_{q_s}^{\kappa(q_s,0)}=1$.

\subsection{Information State and Local Observations}

An information assignment is a partial function
$y:\mathcal{U}\rightharpoonup\{0,1\}$. Its domain
$\operatorname{dom}(y)$ contains the uncertain variables already observed;
$y(z,k)=1$ and $y(z,k)=0$ mean safe and blocked, respectively. An interval
index $(z,k)$ is \emph{certified safe} under $y$ exactly when
$
    p_z^k=1
    \quad\text{or}\quad
    \bigl((z,k)\in\operatorname{dom}(y)\ \text{and}\ y(z,k)=1\bigr)$.
An unobserved uncertain variable cannot be used, and by independence observing
one variable does not change the probability of any other.

For $(q,t)\in V\times\mathcal{T}$, the visibility set
$\mathcal{R}(q,t)\subseteq\mathcal{U}$ contains the variables whose values
can be sensed from $q$ at time $t$, and the newly observed set is
$\mathcal{N}(q,t,y)=\mathcal{R}(q,t)\setminus\operatorname{dom}(y)$.
The sensing model must expose every uncertain status before the
corresponding primitive motion. For an edge $e=(q,r)$ that can depart at
$t$, let
$\mathcal{L}(e,t)=\{(e,\kappa(e,t)),(r,\kappa(r,t+W(e)))\}$ be the pair of
interval indices it uses. We assume
$\mathcal{L}(e,t)\cap\mathcal{U}\subseteq\mathcal{R}(q,t)$
for every valid departure. This includes wait-edge safety and the vertex
occupancy at the end of a wait. The robot senses all variables in
$\mathcal{N}(q,t,y)$ before selecting its next motion, including at
$(q_s,0)$ and after every primitive edge execution.

An observation is an assignment
$\omega:\mathcal{N}(q,t,y)\rightarrow\{0,1\}$. For $j=(z,k)$, use the
shorthands $X_j=X_z^k$ and $p_j=p_z^k$. The observation probability is
\[
    O(\omega\mid q,t,y)
    =
    \prod_{j\in\mathcal{N}(q,t,y)}
      p_j^{\omega(j)}(1-p_j)^{1-\omega(j)}.
\]
The empty product equals one, and every listed variable has probability
strictly between zero and one, so every binary assignment has positive
probability. The post-observation assignment is $y'=y\cup\omega$.

\subsection{States, Actions, and Objective}

A planning state is a tuple $s=(q,t,y)$, where $q\in V$,
$t\in\mathcal{T}$, and the current vertex interval
$(q,\kappa(q,t))$ is certified safe. Every non-goal state is
post-observation and satisfies $\mathcal{N}(q,t,y)=\emptyset$; a goal state
is terminal and requires no observation. An action is an endpoint pair
$a=(q',t')$, with $q'\in V$ and $t<t'\leq T_{\max}$, together with a stored
scheduled motion
$
    \sigma(s,a)
    =
    \bigl((q_0,t_0),e_0,\ldots,e_{m-1},(q_m,t_m)\bigr)$,
with $m\in\mathbb{N}_{>0}$, $(q_0,t_0)=(q,t)$, $(q_m,t_m)=(q',t')$, and, for
every $i<m$, $e_i=(q_i,q_{i+1})\in E$ and $t_{i+1}=t_i+W(e_i)$.

The schedule is applicable under $y$ if (i) every edge interval used at $t_i$
and every destination-vertex interval used at $t_{i+1}$ is certified safe under
$y$; (ii) $\mathcal{N}(q_i,t_i,y)=\emptyset$ and $q_i\neq q_g$ for every
internal arrival index $i\in\{1,\ldots,m-1\}$; and (iii) $q'=q_g$ or
$\mathcal{N}(q',t',y)\neq\emptyset$. Conditions (i) and (ii) together with the
sensing assumption make plan-time certification well defined: an empty newly
observed set at an internal arrival puts every uncertain status that the next
edge uses in $\operatorname{dom}(y)$ already.
Thus, a macro-action stops at the first new observation event or at the
goal. The complete action set $\mathcal{A}(s)$ contains one representative
schedule for every applicable endpoint pair: schedules sharing $(q',t')$ have
the same elapsed cost and next observation law, so one fixed representative
suffices. All applicable endpoint times are retained; earliest-arrival
dominance is not assumed.

Executing $a\in\mathcal{A}(s)$ is deterministic and reaches the
pre-observation tuple
$
    \bar{s}_a=(q',t',y)
$
at cost $c(s,a)=t'-t$. If $q'\neq q_g$, outcome $\omega$ produces the next
decision state
$s_a^\omega=(q',t',y\cup\omega)$. A goal endpoint is terminal and incurs no
stochastic observation; for Bellman bookkeeping below, it has one
deterministic terminal outcome.

A deterministic \emph{macro policy} maps every reachable non-goal planning
state to an action in $\mathcal{A}(s)$. For comparison, the underlying
\emph{primitive model} selects one edge $e=(q,r)\in E$ after mandatory
sensing; it is applicable at $(q,t,y)$ exactly when
$t+W(e)\leq T_{\max}$ and both
$(e,\kappa(e,t))$ and $(r,\kappa(r,t+W(e)))$ are certified safe, and it
incurs cost $W(e)$ and reaches $(r,t+W(e),y)$. A deterministic primitive
policy maps every reachable history ending in a non-goal post-observation
state to an applicable edge; a Markov primitive policy depends only on the
terminal $(q,t,y)$.

A policy of either class is \emph{proper} if every positive-probability
observation branch reaches $q_g$ no later than $T_{\max}$. An improper
policy has value infinity. For a proper policy and $s=(q,t,y)$, let $t_g$
be its goal-arrival time and define
$
    J^\pi(s)=\mathbb{E}[t_g-t]$,
where the expectation uses the independent probabilities of the unresolved
variables. Let $J_{\mathrm{mac}}^*(s)$ and
$J_{\mathrm{prim}}^*(s)$ be the minima over proper macro and primitive
policies, respectively, with an empty minimum equal to infinity. StochSIPP
optimizes the macro value, so the algorithm and analysis below write
$J^*(s)=J_{\mathrm{mac}}^*(s)$. Lemma~\ref{lem:macro_lossless} proves that
this value also equals $J_{\mathrm{prim}}^*(s)$.

Let $y_\emptyset$ denote the empty assignment and let
$\Omega_0^+$ contain all observations of
$\mathcal{N}(q_s,0,y_\emptyset)$. Initial sensing gives
\[
    J_0^*
    =
    \sum_{\omega_0\in\Omega_0^+}
      O(\omega_0\mid q_s,0,y_\emptyset)
      J^*((q_s,0,y_\emptyset\cup\omega_0)).
\]
If the initial visibility set is empty, this reduces to
$J^*((q_s,0,y_\emptyset))$.

%%%%%%%%%%%%%%%%%
\section{StochSIPP}
\label{sec:stochsipp}

This section defines the finite action--observation graph searched by
StochSIPP. SIPP generates certified-safe macro-actions and reusable cost
bounds, while bounded AND--OR search computes the contingent policy.
Exactness is relative to the finite roadmap, horizon, independent Bernoulli
model, and sensing assumptions defined above. StochSIPP takes the interval
partition, probabilities, and visibility sets as input; estimating them from
data or from an external predictor is outside the planning problem studied
here.

\label{subsec:ctp_reduction}
An interval variable plays the decision role of a classical CTP edge status, but
applies to a specified range of edge-departure or vertex-occupancy times, so one
roadmap edge may have different statuses in different intervals. Classical
independent stochastic CTP is the special case with one uncertain interval per
edge spanning the horizon, deterministic vertex occupancies, and incident-edge
sensing. Both cases operate directly on interval-status probabilities, without
maintaining complete obstacle-trajectory hypotheses.

\subsection{Finite Action--Observation Graph} \label{subsec:andor}

Each OR node represents a decision state $s=(q,t,y)$. An action
$a=(q_a,t_a)\in\mathcal{A}(s)$ has the stored schedule
$\sigma(s,a)$ defined above, incurs cost $c(s,a)=t_a-t$, and reaches
$\bar{s}_a=(q_a,t_a,y)$. For a non-goal endpoint, let
$\Omega^+(\bar{s}_a)$ contain every assignment over
$\mathcal{N}(q_a,t_a,y)$. We use the shorthand
$
    O(\omega\mid\bar{s}_a)
    :=
    O(\omega\mid q_a,t_a,y)$.
Each $\omega\in\Omega^+(\bar{s}_a)$ produces the successor
$s_a^\omega=(q_a,t_a,y\cup\omega)$. A goal endpoint has
$\Omega^+(\bar{s}_a)=\{\omega_g\}$,
$O(\omega_g\mid\bar{s}_a)=1$, and terminal successor
$s_a^{\omega_g}=(q_g,t_a,y)$.

The pair $(s,a)$ is an AND node with one outgoing arc for every outcome.
OR nodes are cached by $(q,t,y)$, AND nodes are cached by $(s,a)$, and every
cached node records all parents. Since each action satisfies $t_a>t$, the
resulting finite graph is acyclic.

Let $Q^*(s,a)$ be the optimal cost among policies whose first action at $s$
is $a$. The Bellman equations are
\begin{align}
    Q^*(s,a)
    &=
    c(s,a)
    +
    \sum_{\omega\in\Omega^+(\bar{s}_a)}
      O(\omega\mid\bar{s}_a)J^*(s_a^\omega),
    \label{eq:bellman-q}\\
    J^*(s)
    &=
    \begin{cases}
        0, & q=q_g,\\
        \min_{a\in\mathcal{A}(s)}Q^*(s,a),
            & q\neq q_g,\ \mathcal{A}(s)\neq\emptyset,\\
        \infty, & q\neq q_g,\ \mathcal{A}(s)=\emptyset.
    \end{cases}
    \label{eq:bellman-j}
\end{align}
All outcome probabilities in the sum are positive. We use extended-real
arithmetic, so a positive-probability infinite successor makes the action
value infinite.

Initial sensing is represented by a zero-cost chance root $r$ with one child
$(q_s,0,y_\emptyset\cup\omega_0)$ per $\omega_0\in\Omega_0^+$, of arc
probability $O(\omega_0\mid q_s,0,y_\emptyset)$. A solution subgraph retains all
initial outcomes, one action at every reachable non-goal OR node, and all
outcomes of every retained AND node.

\subsection{SIPP-Based Admissible Bounds} \label{subsec:sipp_bounds}

This subsection constructs bounds that can be cached independently of $y$.
The \emph{optimistic roadmap} treats every uncertain interval as safe, even
if $y$ records it as blocked. The \emph{robust roadmap} treats every
uncertain interval as blocked, even if $y$ records it as safe. Intervals
with probability zero or one retain their deterministic status.

Let $A_q^{\mathrm{opt}}(t)$ and $A_q^{\mathrm{rob}}(t)$ be the earliest
absolute goal-arrival times returned by any-start-time SIPP on the two
roadmaps~\cite{thomas2023anystarttime}. The query admits the already occupied
start pair $(q,t)$ because it is certified safe in every reachable state.
After that pair, a robust path may use only probability-one edge and vertex
intervals. Define
$
    h^{\mathrm{opt}}(q,t)
    =
    A_q^{\mathrm{opt}}(t)-t$,
$
    h^{\mathrm{rob}}(q,t)
    =
    A_q^{\mathrm{rob}}(t)-t$,
where an infeasible query has value $\infty$. Every feasible policy branch
is present in the optimistic roadmap, and a finite robust route is a policy
witness that ignores observations and remains safe for every outcome, so
these two quantities bracket $J^*(s)$ at every reachable state
(Proposition~\ref{prop:admissible_bounds}).

An unexpanded non-goal OR node $s=(q,t,y)$ is initialized with
$
    \ell(s)=h^{\mathrm{opt}}(q,t)$,
$
    u(s)=h^{\mathrm{rob}}(q,t)$.
Goal nodes have $\ell(s)=u(s)=0$. If
$h^{\mathrm{opt}}(q,t)=\infty$, both bounds equal infinity and the node is a
solved dead end.

After generating an action, its bounds are
\begin{align}
    \underline{Q}(s,a)
    &=
    c(s,a)+
    \sum_{\omega\in\Omega^+(\bar{s}_a)}
      O(\omega\mid\bar{s}_a)\ell(s_a^\omega),
    \label{eq:action-lower}\\
    \overline{Q}(s,a)
    &=
    c(s,a)+
    \sum_{\omega\in\Omega^+(\bar{s}_a)}
      O(\omega\mid\bar{s}_a)u(s_a^\omega).
    \label{eq:action-upper}
\end{align}
Let $\mathcal{A}_{\mathrm{act}}(s)\subseteq\mathcal{A}(s)$ be the actions
not pruned at $s$. An expanded nonterminal node is backed up as
\begin{align}
    \ell(s)
    &=
    \max\!\left\{
      h^{\mathrm{opt}}(q,t),
      \min_{a\in\mathcal{A}_{\mathrm{act}}(s)}
        \underline{Q}(s,a)
    \right\},
    \label{eq:or-lower}\\
    u(s)
    &=
    \min\!\left\{
      h^{\mathrm{rob}}(q,t),
      \min_{a\in\mathcal{A}_{\mathrm{act}}(s)}
        \overline{Q}(s,a)
    \right\}.
    \label{eq:or-upper}
\end{align}
The minimum of an empty set is infinity. A probability-aware expected-value estimate may order expansions, but it is not used as a bound, pruning threshold, or certificate.

\subsection{Bounded Exact Search} \label{subsec:exact_search}

StochSIPP uses a CAO*-style best-first search~\cite{cao_paper}. At an
expanded state $s$, every finite term of $u(s)$ in Eq.~\ref{eq:or-upper} is
the value of some feasible policy, so $u(s)$ doubles as a feasible-policy
witness; write $B(s)=u(s)$. An action is pruned only when
\begin{equation}
    \underline{Q}(s,a)>B(s).
    \label{eq:prune}
\end{equation}
The strict inequality preserves actions tied with the witness.

An AND node is solved exactly when all its positive-probability OR children
are solved. An expanded OR node is solved when it has a solved active action
$a^\star$ such that
\begin{equation}
    \underline{Q}(s,a^\star)
    =
    \overline{Q}(s,a^\star)
    \leq
    \min_{\substack{a\in\mathcal{A}_{\mathrm{act}}(s)\\a\neq a^\star}}
      \underline{Q}(s,a).
    \label{eq:or-solved}
\end{equation}
The algorithm stores $a^\star$ as the certifying action. A non-goal node is
solved at infinity either immediately when
$h^{\mathrm{opt}}(q,t)=\infty$, or after expansion when the complete
generator returns $\mathcal{A}(s)=\emptyset$. The chance root is solved
exactly when all its children are solved; its lower and upper bounds are the
corresponding probability-weighted sums. Multiple certifying actions are resolved by the same total order used for action selection. A goal OR node is created expanded and solved with value zero; every other OR node is created unexpanded and unsolved with the fixed bounds above.

The current lower-bound solution graph contains all children of the chance
root and of each selected AND node, selecting at every expanded unsolved OR
node an active action minimizing $\underline{Q}(s,a)$. Each iteration expands
one unexpanded OR node on this graph and then backs up AND-node bounds and
solved flags, active sets, witnesses, OR-node bounds, selected and certifying
actions, and the chance-root weighted bounds. The backup follows every parent
link in reverse topological order until no stored value changes, so merging
identical states does not omit a backup.

\begin{algorithm}[tb]
\caption{Bounded Exact Search for StochSIPP}
\label{alg:stochsipp}
{\small
\begin{algorithmic}[1]
\REQUIRE Finite model defined in the problem formulation, complete macro-action
generator, and $h^{\mathrm{opt}},h^{\mathrm{rob}}$
\ENSURE An optimal contingent policy or failure
\STATE Initialize OR cache $C_O$ by $(q,t,y)$ and AND cache $C_A$ by $(s,a)$
\STATE Create chance root $r$ with its initial children, arcs, and parent
links, then back up $r$
\WHILE{$r$ is not solved}
    \STATE Select, by a fixed deterministic priority, an unexpanded unsolved
    OR node $v$ on the lower-bound solution graph; mark $v$ expanded
    \STATE Generate complete $\mathcal{A}(s_v)$ and schedules $\sigma(s_v,a)$
    \IF{$\mathcal{A}(s_v)=\emptyset$}
        \STATE Set $\ell(s_v)=u(s_v)=\infty$ and mark $v$ solved
    \ELSE
        \STATE $\mathcal{A}_{\mathrm{act}}(s_v)\gets\mathcal{A}(s_v)$
        \FOR{each $a\in\mathcal{A}(s_v)$}
            \STATE Retrieve or create $(s_v,a)$ in $C_A$ with its schedule,
            cost, arc from $v$, and parent link
            \STATE Retrieve or create each $s_a^\omega$,
            $\omega\in\Omega^+(\bar{s}_a)$, in $C_O$ with its arc probability
            $O(\omega\mid\bar{s}_a)$ and parent link
        \ENDFOR
    \ENDIF
    \STATE Back up $v$ and all ancestors to a fixed point using
    Eqs.~\ref{eq:action-lower}--\ref{eq:or-solved}
\ENDWHILE
\STATE \textbf{if} $\ell(r)=\infty$ \textbf{then} report failure
\STATE \textbf{else} return $s\mapsto(a^\star,\sigma(s,a^\star))$ at every
nonterminal OR node reachable from $r$ via certifying actions and all
outcome arcs
\end{algorithmic}
}
\end{algorithm}

\paragraph{Policy execution.}
The solved subgraph is executed without online replanning: the robot senses at
$(q_s,0)$, follows the corresponding chance-root arc, and then at each OR state
executes the stored schedule of its certifying action, senses at the endpoint,
extends $y$, and follows the outcome arc until it reaches the goal. A sensing
error, an unmodeled status, or a timing deviation violates the policy
assumptions and requires a separate safe-stop or replanning mechanism.

%%%%%%%%%%%%%%%
\section{Theoretical Analysis}
\label{sec:theory}

This section validates the state abstraction, macro-actions, bounds, pruning,
safety, and finite-roadmap optimality under the independent interval model.
Complete proofs of all formal results are in the supplementary material.

\subsection{Computational Hardness}
\label{subsec:hardness}

\begin{theorem}[PSPACE-hardness]
\label{thm:hardness}
When durations, $T_{\max}$, interval endpoints, rational probabilities, and
a rational threshold $C$ use binary encoding, deciding whether
$J_0^*\leq C$ is PSPACE-hard. This holds with deterministic vertex occupancies
and one time-invariant independent Bernoulli variable per uncertain edge.
\end{theorem}

\begin{proof}[Proof sketch]
Reduce from stochastic CTP with directed edges and independent statuses,
which is PSPACE-complete, as is deciding the optimal expected cost of the
constructed instances~\cite{fried2013complexity}. Rational costs are scaled to
positive integer durations while preserving the expected-cost threshold: a
cycle-free policy has at most $H=n3^m$ traversals, where $n$ is the number of
vertices and $m$ the number of uncertain edges, so the additive traversal term
can be made smaller than the minimum rational expected-cost gap. Represent
each uncertain CTP edge by one time-invariant interval variable, expose
incident variables at the current vertex, and set a finite horizon large
enough for every cycle-free policy. All constructed numerical values are
binary encoded, and each interval is stored by its endpoints. This gives a
cost- and observation-preserving correspondence between proper policies. The
complete scaling argument and encoding bound are given in the supplementary
material.
\end{proof}

\subsection{Correctness, Bounds, and Pruning}
\label{subsec:correctness}

\begin{lemma}[Information-state sufficiency]
\label{lem:sufficiency}
Two histories ending in the same decision state $(q,t,y)$ have the same
applicable actions, action costs, and future observation probabilities. If a
proper primitive policy exists, an optimal deterministic Markov primitive
policy exists on these states.
\end{lemma}

\begin{proof}[Proof sketch]
After conditioning on $y$, mutual independence leaves every unobserved
variable with its original Bernoulli probability, so certification, motion,
and the next visible set depend only on $(q,t,y)$ and the selected action.
Since every primitive step advances time in a finite horizon, backward
induction yields an optimal deterministic Markov policy whenever a proper
policy exists.
\end{proof}

\begin{lemma}[Macro-action losslessness]
\label{lem:macro_lossless}
Assume the macro-action generator includes every applicable endpoint-time
pair. From any reachable state, every proper deterministic Markov primitive
policy induces a StochSIPP macro-action policy with the same goal-arrival time
for every complete status assignment permitted by the model and consistent with that state.
Conversely, expanding each stored schedule induces a primitive policy with the
same pathwise arrival time.
\end{lemma}

\begin{proof}[Proof sketch]
Partition each primitive-policy execution at its first new observation or
goal arrival. Complete endpoint generation represents every resulting
certified segment as a macro-action, and replacing a segment by another stored
schedule with the same endpoint and time preserves both cost and the next
observation law. Expanding each stored schedule recovers a primitive
controller with the same pathwise arrival time, so
$J_{\mathrm{mac}}^*(s)=J_{\mathrm{prim}}^*(s)$.
\end{proof}

Complete action and outcome generation additionally makes proper deterministic
Markov macro-policies and proper solution subgraphs interchangeable: at every
reachable state they have the same conditional expected cost, by induction on
the remaining horizon, since Eq.~\ref{eq:bellman-q} assigns the selected action
exactly the cost of the policy recursion. We refer to this as
\emph{policy correspondence}.

\label{subsec:bound_pruning}

\begin{proposition}[Admissible SIPP and action bounds]
\label{prop:admissible_bounds}
For every reachable state $s=(q,t,y)$,
$
    h^{\mathrm{opt}}(q,t)
    \leq J^*(s)
    \leq h^{\mathrm{rob}}(q,t)$.
If every successor satisfies
$\ell(s_a^\omega)\leq J^*(s_a^\omega)\leq u(s_a^\omega)$, then every
generated action satisfies
$
    \underline{Q}(s,a)
    \leq Q^*(s,a)
    \leq \overline{Q}(s,a)$.
\end{proposition}

\begin{proof}[Proof sketch]
Every branch of every proper policy is feasible in the optimistic roadmap, so
its earliest arrival lower-bounds that branch and hence $J^*(s)$, while a
robust route uses only certified probability-one intervals and is a feasible
witness. Multiplying successor bounds by nonnegative outcome
probabilities and adding $c(s,a)$ gives the action inequalities.
\end{proof}

\begin{proposition}[Sound pruning and OR backups]
\label{prop:sound_pruning}
Suppose the node and action bounds are valid, and every finite upper bound
is backed by a feasible witness policy. If
$\underline{Q}(s,a)>B(s)$, then $a$ is not optimal. After any sequence of
such removals, Eqs.~\ref{eq:or-lower} and~\ref{eq:or-upper} satisfy
$
    \ell(s)\leq J^*(s)\leq u(s)$.
\end{proposition}

\begin{proof}[Proof sketch]
Every finite term in $B(s)$ is backed by a feasible policy, so
$J^*(s)\leq B(s)$. For a pruned action,
$
Q^*(s,a)\geq\underline{Q}(s,a)>B(s)\geq J^*(s)$,
and the action cannot be optimal. An optimal action therefore remains
active, and taking minima of valid active-action bounds preserves the OR-node
bounds.
\end{proof}

\subsection{Safety and Finite-Roadmap Optimality}
\label{subsec:optimality}

\begin{theorem}[Primitive-step safety]
\label{thm:safety}
Assume every interval with safety probability one is truly safe, observations report realized interval statuses exactly, and execution follows the planned timing.
Every execution that starts at $q_s$ and uses only applicable StochSIPP
actions is collision-free.
\end{theorem}

\begin{proof}
Maintain two invariants: $y$ agrees with every status observed in the
current episode, and the robot safely occupies its current vertex. They hold
initially because $y_\emptyset$ contains no claims and the start occupancy
has probability one. Exact sensing preserves the first invariant.

Every interval used by an applicable schedule is certified safe. A probability-one variable is safe by assumption, and an uncertain certified variable was observed with value one. Hence every used edge execution and destination occupancy is safe in the episode. The lookahead sensing condition exposes these variables before departure. For a wait, the self-loop event covers the waiting period and the destination event covers its endpoint. Induction over all primitive steps preserves safe occupancy and proves collision freedom.
\end{proof}

\begin{theorem}[Finite-roadmap optimality]
\label{thm:optimality}
Under the assumptions of Theorem~\ref{thm:safety}, additionally assume the
interval probabilities are exact and independent, action and outcome
generation are complete, and Algorithm~\ref{alg:stochsipp} uses the stated
caches, backups, solved tests, and strict pruning rule. Then the algorithm
terminates after finitely many expansions. If a proper primitive policy
exists on the roadmap, it returns a collision-free policy with minimum
expected goal-arrival time among all such policies. Otherwise, it returns
failure.
\end{theorem}

\begin{proof}[Proof sketch]
There are at most $|V|(T_{\max}+1)3^M$ OR-state keys and finitely many endpoint
pairs per state, and every action advances time, so the cached graph is a
DAG. While its root is unsolved, the selection rule reaches an unexpanded OR
node, and strict sound pruning retains an optimal action. Valid
bounds and exact solved-node backups then establish the exact root value by
reverse induction, and Lemma~\ref{lem:macro_lossless} with policy correspondence
transfers that solution to a primitive policy, whose execution is safe by
Theorem~\ref{thm:safety}.
\end{proof}

\paragraph{Scope of the guarantee.}
The result is exact for the supplied finite roadmap, horizon, independent
interval variables, and sensing model, and does not imply global optimality in
the continuous configuration space. Correlated statuses, incorrectly merged time
ranges, sensing errors, or timing deviations violate the model assumptions.
Incorrect probabilities invalidate expected-cost optimality but not safety,
provided every probability-one interval is truly safe and the sensing and timing assumptions hold.

%%%%%%%%%%%%%%%%%
\section{Experiments}
\label{sec:exp}

We ask three questions. Does conditioning a route choice on an observed gate
beat committing to a fixed path? How much search do the bounds and the
ordering estimate save? And how does explicit joint-outcome branching scale
with the number of statuses that are ambiguous at once? Gazebo warehouse experiments additionally provide a quantitative evaluation of policy execution in a more realistic simulated environment. Per-seed controlled-roadmap results and the full quantitative Gazebo comparison are in the supplement.

\subsection{Setup and Scope of the Empirical Claims}

The controlled study uses fifty warehouse-like, 50-vertex roadmaps in an
$800\times800$ workspace with five static and three moving obstacles. Each
moving obstacle independently follows one of two candidate trajectories with
equal prior probability, giving eight trajectory-level worlds per map and
induced edge--time marginals. Every map pairs a short route containing one permanently ambiguous gate ($p=0.5$) with a long route that is always safe, and
the gate is observed before the route choice. StochSIPP computes one
solution subgraph per map and replays it over 1,000 sampled worlds.

Seven open-loop baselines each commit to one path before any trial and never
revise it. Five discount edge cost by safety probability and admit intervals by
a threshold that \emph{det-A*}, \emph{stoch-det-A*}, \emph{0.5-A*},
\emph{0.7-A*}, and \emph{1.0-A*} set at every edge, nonzero probability,
$0.5$, $0.7$, and certainty, respectively. \emph{Most-likely} plans against the
most probable joint obstacle outcome, and \emph{reactive} commits to the
shortest route but halts once it observes a blocked edge, so it never
collides but does not always arrive. A trial succeeds when the robot reaches the
goal without entering the $\texttt{ROB\_RAD}+\texttt{OBST\_RAD}=1.0$ clearance
used to construct the safe intervals; for \emph{reactive}, whose failures are
halts rather than collisions, success is the goal-reaching rate.
Figure~1 illustrates the Gazebo warehouse experiment, which is evaluated over 1,000 trials in two worlds using a $1.2$~m collision radius.
\todo{Specify the probability-weighted A* edge-cost formula, the Full-LRTA
refinement rule, the probability-aware ordering score, hardware and software,
timing repetitions, the wall-clock timeout, and every random seed. Document how
sampled trajectories produce the interval partition, episode-fixed statuses,
visibility sets, and observations used at execution.}

Three limits bound what these numbers establish. Trajectory sampling generally
correlates statuses across edges and times, so Theorem~\ref{thm:optimality} does
not apply to the sampled distribution; the evaluations therefore measure the practical value of contingent route selection when the supplied edge--time marginals arise from correlated trajectory-level worlds, rather than claiming expected-cost optimality under that evaluation distribution. Zero observed collisions is consistent
with, but does not verify, the conditions of Theorem~\ref{thm:safety}. Recorded
arrival includes collided runs, so an unsafe baseline's value is not a
proper-policy expected arrival time.

\subsection{Contingency Against Fixed Paths}

\begin{table}[t]
\centering
\small
\setlength{\tabcolsep}{3pt}
\begin{tabular}{lccccc}
\toprule
& \multicolumn{3}{c}{50 risky maps} & \multicolumn{2}{c}{Gated} \\
\cmidrule(lr){2-4}\cmidrule(lr){5-6}
Method & Arr. & All-Pass & Succ.\% & 2-rt\% & 3-rt\% \\
\midrule
StochSIPP & 11.6 & 7.9 & $100.0$ & $100.0$ & $100.0$ \\
1.0-A*, 0.7-A* & 14.1 & 14.1 & $100.0$ & $0.0$ & $0.0$ \\
most-likely & 14.1 & 14.1 & $100.0$ & $50.3$ & $66.4$ \\
det-A*, stoch-det-A*, & \multirow{2}{*}{6.9} & \multirow{2}{*}{6.9}
  & \multirow{2}{*}{$49.7$} & \multirow{2}{*}{$49.7$}
  & \multirow{2}{*}{$65.7$} \\
\quad 0.5-A*, reactive & & & & & \\
\bottomrule
\end{tabular}
\caption{Mean arrival and success (\%), 1,000 trials per scenario; grouped
methods recorded identical values. \emph{All-Pass} restricts arrival to worlds
where every method succeeds, which selects for worlds in which the gate is open. On the gated
scenarios every succeeding baseline arrives at $6.0$ and StochSIPP at
$7.0$; $0.0$ means no plan was returned.}
\label{tab:results}
\end{table}

On every map the baselines collapse into two clusters that each commit to an
identical route regardless of the admission rule that produced it: a cautious
cluster (\emph{1.0-A*}, \emph{0.7-A*}, \emph{most-likely}) succeeding in every
trial and a risky cluster (\emph{det-A*}, \emph{stoch-det-A*}, \emph{0.5-A*},
\emph{reactive}) succeeding in $49.7\%$ (Table~\ref{tab:results}). We summarize the per-map paired differences as $\Delta=\mathrm{StochSIPP}-\mathrm{baseline}$, reported as a mean with a 95\% interval. Against the cautious cluster $\Delta$success is exactly
$0$ with zero width while $\Delta$arrival is $-2.6\pm0.4$: StochSIPP
matches their perfect success at $7\%$--$31\%$ lower mean arrival ($17\%$ on
average), taking the short route whenever the gate is reported clear. Against
the risky cluster $\Delta$success is a constant $+50.3$ points, again with zero
width, and $\Delta$arrival is $+4.6\pm0.4$: their shorter times are bought by
gambling on the same gate, half the time by colliding or, for \emph{reactive},
by halting short of the goal. Planning costs $0.033\pm0.004$~s more than a
single-shot path lookup.

The gated scenarios sharpen the comparison. With one (\emph{2-route}) and two
(\emph{3-route}) ambiguous gates between start and goal, no route clears the
$0.7$ or $1.0$ threshold, so \emph{0.7-A*} and \emph{1.0-A*} return no plan at
all rather than a risky one, while StochSIPP reaches the goal in every
trial by disambiguating each gate in turn. \emph{Most-likely} no longer belongs to the cautious cluster in these scenarios because selecting each obstacle's more probable hypothesis independently does not guarantee a jointly safe route when gates are involved.

\subsection{Combined Search-Aid Ablation}

The \emph{full} configuration combines the admissible bounds, the probability-aware expected-value ordering estimate, and cached arrival-time functions. \emph{No-heur} removes these mechanisms, while \emph{full-lrta} augments \emph{full} with an LRTA-style learning mechanism that refines cached estimates along explored branches using backed-up search values \cite{bulitko2006learning}. Further implementation details are provided in the supplementary material.

\begin{table}[t]
\centering
\small
\setlength{\tabcolsep}{3pt}
\begin{tabular}{lccc}
\toprule
\multicolumn{4}{c}{(a) Ablation: $\Delta=\text{\emph{full}}-\text{variant}$} \\
Variant & $\Delta$Expands & $\Delta$Nodes & $\Delta$Solve (s) \\
\midrule
no-heur & $-21210\pm18712$ & $-62320\pm47816$ & $-4.15\pm2.71$ \\
full-lrta & $0\pm0$ & $0\pm0$ & $-1.12\pm0.44$ \\
\bottomrule
\end{tabular}
\begin{tabular}{rrrrr}
\multicolumn{5}{c}{} \\
\toprule
\multicolumn{5}{c}{(b) Simultaneous ambiguity} \\
Statuses & Build (s) & Solve (s) & Expands & Nodes \\
\midrule
2 & 0.003 & 0.002 & 29 & 116 \\
3 & 0.005 & 0.016 & 175 & 796 \\
4 & 0.006 & 0.179 & 1,401 & 7,140 \\
5 & 0.010 & 2.459 & 14,011 & 79,084 \\
6 & 0.012 & 39.639 & 168,133 & 1,041,172 \\
\bottomrule
\end{tabular}
\caption{Search effort. (a) Paired differences over seven scenarios (five random
seeds plus the two gated ones), mean $\pm$ half-width of the $95\%$ interval;
negative favors \emph{full}. (b) Growth with simultaneously ambiguous statuses.}
\label{tab:effort}
\end{table}

All three configurations return the same root value on every scenario
($J_0^*=7.0$ on seeds 3 and 7, $J_0^*=6.0$ on seeds 8, 10, and 22, and
identically on both gated scenarios), so the search aids change only how much of
the graph is searched. Every column of Table~\ref{tab:effort}(a) favors
\emph{full} with an interval clear of zero. The intervals are wide because
effort spans three orders of magnitude across scenarios: \emph{full} expands
3--7 OR nodes on the five random maps but 3,837 on \emph{2-route} and 13,585 on
\emph{3-route}, and its advantage over \emph{no-heur} falls from three or four
orders of magnitude on the random maps to $16\times$ and under $4\times$ on the
gated ones. Because several mechanisms change together, this supports their
combined benefit rather than an isolated claim about the bounds.
\emph{Full-lrta} expands exactly the same nodes as \emph{full} everywhere and is
only slower, so its refinement yields no search reduction at these sizes.

\subsection{Scalability with Simultaneous Ambiguity}

The stress test isolates the number of statuses ambiguous at one decision state.
A safe edge joins start and goal, and $k$ independent moving obstacles
each add one permanently ambiguous dead-end edge from the start, so the roadmap
has $k+2$ vertices while the solver branches on $k$ statuses at once.

Table~\ref{tab:effort}(b) shows solve time growing far faster than the roadmap, with increasing multipliers ($8\times$, $11\times$, $14\times$, $16\times$); expansions and nodes grow similarly. This stress test varies simultaneous ambiguity at one decision state, not total uncertainty across the roadmap, so the growth reflects joint-outcome branching rather than roadmap size. Build time never exceeds $0.012$~s, and the fixed-path baselines plan in under a millisecond. Settings beyond six statuses did not finish within the timeout. Five points show rapid, compounding growth but cannot establish an asymptotic rate.
\todo{Record which settings beyond six statuses were attempted and the
wall-clock timeout used.}

%%%%%%%%%%%%%%%%%%%%%%%%%%%%%%%%
%%%%%%%%%%%%%%%%%
\section{Conclusion}
\label{sec:conc}

StochSIPP searches a finite action--observation graph over independent Bernoulli interval statuses. Under the assumptions of Theorem~\ref{thm:safety}, its resulting policy is provably collision-free, and it is expected-cost optimal when the interval probabilities are correct and independent. On fifty single-gate maps, it matched the safe baselines' $100\%$ observed success at $7\%$--$31\%$ lower mean arrival and reached the goal in scenarios where conservative baselines returned no plan. Extending StochSIPP to correlated uncertainty and developing approximate methods for large joint observation spaces are important directions for future work.

% \section*{Acknowledgments}

\bibliography{aaai2027}

\section{Purpose and Scope}

This supplement provides the complete proofs and the disaggregated
experimental results referenced by the main paper. The proofs concern the
finite-roadmap model with mutually independent, episode-fixed Bernoulli
interval variables, exact sensing, and planned timing. The experimental
section separately identifies which reported study satisfies that
factorization assumption.

\section{Proof Setup}

This section restates the notation and search equations needed to read the
proofs independently of the main paper. Let $\mathcal{U}=\{1,\ldots,M\}$ be
the set of uncertain edge- and vertex-interval variables. Variable
$X_j\in\{0,1\}$ has supplied safety probability
$p_j=\Pr[X_j=1]$, and the variables are mutually independent. A decision
state is $s=(q,t,y)$, where $q$ is a roadmap vertex, $t$ is an integer time,
and the partial map $y$ records the observed variables and their statuses.
The set $\operatorname{dom}(y)$ contains the variables recorded by $y$.
The finite planning horizon is $T_{\max}$, and $q_s$ and $q_g$ are the start
and goal vertices.

An applicable macro-action $a=(q_a,t_a)\in\mathcal{A}(s)$ stores a
certified-safe schedule $\sigma(s,a)$, incurs
$c(s,a)=t_a-t$, and reaches the pre-observation state
$\bar{s}_a=(q_a,t_a,y)$. The complete generator includes every applicable
endpoint-time pair. Let $\Omega^+(\bar{s}_a)$ contain every
positive-probability assignment $\omega$ to the variables newly visible at
the endpoint. The probability of an assignment is
$O(\omega\mid\bar{s}_a)$, and its successor is
$s_a^\omega=(q_a,t_a,y\cup\omega)$. A proper policy reaches $q_g$ within
the horizon for every positive-probability outcome sequence. Its optimal
remaining expected arrival time is $J^*(s)$, and $Q^*(s,a)$ is the optimal
cost conditioned on selecting $a$ first. The value $J_0^*$ is the
probability-weighted value after initial sensing. The Bellman equations are
\begin{align}
    Q^*(s,a)
    &=
    c(s,a)
    +
    \sum_{\omega\in\Omega^+(\bar{s}_a)}
      O(\omega\mid\bar{s}_a)J^*(s_a^\omega),
    \label{eq:bellman-q}\\
    J^*(s)
    &=
    \begin{cases}
        0, & q=q_g,\\
        \min_{a\in\mathcal{A}(s)}Q^*(s,a),
            & q\neq q_g,\ \mathcal{A}(s)\neq\emptyset,\\
        \infty, & q\neq q_g,\ \mathcal{A}(s)=\emptyset.
    \end{cases}
    \label{eq:bellman-j}
\end{align}

The optimistic SIPP relaxation treats every uncertain interval as safe,
whereas the robust relaxation treats every uncertain interval as blocked
after the certified current occupancy. Their remaining-time values are
$h^{\mathrm{opt}}(q,t)$ and $h^{\mathrm{rob}}(q,t)$, respectively. Each OR
node stores lower and upper bounds $\ell(s)$ and $u(s)$. The corresponding
action bounds are
\begin{align}
    \underline{Q}(s,a)
    &=
    c(s,a)+
    \sum_{\omega\in\Omega^+(\bar{s}_a)}
      O(\omega\mid\bar{s}_a)\ell(s_a^\omega),
    \label{eq:action-lower}\\
    \overline{Q}(s,a)
    &=
    c(s,a)+
    \sum_{\omega\in\Omega^+(\bar{s}_a)}
      O(\omega\mid\bar{s}_a)u(s_a^\omega).
    \label{eq:action-upper}
\end{align}
Let $\mathcal{A}_{\mathrm{act}}(s)\subseteq\mathcal{A}(s)$ contain the
actions not pruned at $s$. The OR backups are
\begin{align}
    \ell(s)
    &=
    \max\!\left\{
      h^{\mathrm{opt}}(q,t),
      \min_{a\in\mathcal{A}_{\mathrm{act}}(s)}
        \underline{Q}(s,a)
    \right\},
    \label{eq:or-lower}\\
    u(s)
    &=
    \min\!\left\{
      h^{\mathrm{rob}}(q,t),
      \min_{a\in\mathcal{A}_{\mathrm{act}}(s)}
        \overline{Q}(s,a)
    \right\}.
    \label{eq:or-upper}
\end{align}
The feasible-policy witness is
\begin{equation}
    B(s)
    =
    \min\!\left\{
      h^{\mathrm{rob}}(q,t),
      \min_{a\in\mathcal{A}_{\mathrm{act}}(s)}
        \overline{Q}(s,a)
    \right\}.
    \label{eq:witness}
\end{equation}
An action is pruned only if
\begin{equation}
    \underline{Q}(s,a)>B(s).
    \label{eq:prune}
\end{equation}
An expanded OR node is solved when it has a solved active action $a^\star$
such that
\begin{equation}
    \underline{Q}(s,a^\star)
    =
    \overline{Q}(s,a^\star)
    \leq
    \min_{\substack{a\in\mathcal{A}_{\mathrm{act}}(s)\\a\neq a^\star}}
      \underline{Q}(s,a).
    \label{eq:or-solved}
\end{equation}
OR nodes are cached by $(q,t,y)$, AND nodes by $(s,a)$, and all parent
links are retained. Every action strictly advances time.

\begin{algorithm}[tb]
\caption{Bounded Exact Search for StochSIPP}
\label{alg:stochsipp}
{\small
\begin{algorithmic}[1]
\REQUIRE Finite model defined above, complete macro-action generator, and
$h^{\mathrm{opt}},h^{\mathrm{rob}}$
\ENSURE An optimal contingent policy or failure
\STATE Initialize OR cache $C_O$ by $(q,t,y)$ and AND cache $C_A$ by $(s,a)$
\STATE Create chance root $r$, its initial children, arcs, and parent links
\STATE Back up $r$
\WHILE{$r$ is not solved}
    \STATE $\mathcal{F}\gets$ unexpanded, unsolved OR nodes on the
    lower-bound solution graph
    \STATE Select $v\in\mathcal{F}$ by a fixed deterministic priority
    \STATE Let $s_v$ be the state stored at $v$; mark $v$ expanded
    \STATE Generate complete $\mathcal{A}(s_v)$ and schedules $\sigma(s_v,a)$
    \IF{$\mathcal{A}(s_v)=\emptyset$}
        \STATE Set $\ell(s_v)=u(s_v)=\infty$ and mark $v$ solved
    \ELSE
        \STATE Set $\mathcal{A}_{\mathrm{act}}(s_v)\gets\mathcal{A}(s_v)$
        \FOR{each $a\in\mathcal{A}(s_v)$}
            \STATE Retrieve or create $(s_v,a)$ in $C_A$; store its
            schedule, cost, arc from $v$, and parent link
            \FOR{each $\omega\in\Omega^+(\bar{s}_a)$}
                \STATE Retrieve or create $s_a^\omega$ in $C_O$; store its
                arc probability $O(\omega\mid\bar{s}_a)$ and parent link
            \ENDFOR
        \ENDFOR
    \ENDIF
    \STATE Back up $v$ and all ancestors to a fixed point using
    Eqs.~\ref{eq:action-lower}--\ref{eq:or-solved}
\ENDWHILE
\IF{$\ell(r)=u(r)=\infty$}
    \STATE Report failure
\ELSE
    \STATE Return $s\mapsto(a^\star,\sigma(s,a^\star))$ at every
    nonterminal OR node reached by following certifying actions and all
    outcome arcs from $r$
\ENDIF
\end{algorithmic}
}
\end{algorithm}

\section{Complete Proofs}

This section gives the complete proofs of every theorem, lemma, and
proposition stated in the main paper. The statements are restated verbatim and
carry the same numbers as in the main paper.

\subsection{Computational Hardness}

\begin{theorem}[PSPACE-hardness]
\label{thm:hardness}
When durations, $T_{\max}$, interval endpoints, rational probabilities, and
a rational threshold $C$ use binary encoding, deciding whether
$J_0^*\leq C$ is PSPACE-hard. This holds with deterministic vertex occupancies
and one time-invariant independent Bernoulli variable per uncertain edge.
\end{theorem}

\begin{proof}
Fried et al.~\cite{fried2013complexity} prove that stochastic CTP is
PSPACE-complete and that it remains so with directed edges and no dependencies
between edge statuses (their Theorem~2 and Corollary~2); for the same
construction, determining the expected cost of the optimal policy is
PSPACE-hard (their Corollary~1), which gives the expected-cost decision
version. Their construction uses only polynomially many bits for all weights
and probabilities, so all rational numbers in the decision instance use binary
encoding. A negative threshold is a trivial no-instance because
all costs are nonnegative, so assume $C\geq0$. We first convert the rational
edge costs to positive integer durations without changing the decision
answer.
Scale the edge costs and threshold so that every edge cost $w(e)$ and the
threshold $C$ are nonnegative integers. Let $m$ be the number of uncertain
edges, let $n$ be the number of vertices, and let $D$ be the product of the
denominators of their rational blocking probabilities. For any proper
deterministic policy, its expected cost has denominator dividing $D$.
Hence, if that cost is greater than $C$, it is at least $C+1/D$.

A CTP information state consists of a vertex and one of three labels for
each uncertain edge, so there are at most
$H=n3^m$ such states. Any cycle returning to the same information state
reveals no new status and can be removed without increasing cost. Thus, a
proper deterministic policy can be chosen cycle-free, with at most $H$ edge
traversals per branch. Set $K=HD+1$, replace every edge cost by the positive
integer
\[
    W(e)=K w(e)+1,
\]
and replace the decision threshold by $C'=KC+H$. For a proper deterministic
cycle-free policy $\pi$ and the sampled edge-status assignment, let
$C_\pi$ and $N_\pi$ be the resulting original path cost and traversal
count, respectively. These are random variables, and the transformed
expected cost is
\[
    K\,\mathbb{E}[C_\pi]+\mathbb{E}[N_\pi].
\]
If $\mathbb{E}[C_\pi]\leq C$, this value is at most $KC+H=C'$. If
$\mathbb{E}[C_\pi]>C$, it is greater than or equal to
$K(C+1/D)>KC+H=C'$. The transformation therefore preserves the answer.
The binary encodings of $H$, $D$, and $K$ have polynomial length.

Construct the StochSIPP instance on the transformed CTP graph. Assign every
vertex and wait self-loop probability one. Assign each uncertain CTP edge
one interval spanning all valid departure times, with its CTP safety
probability. Use time-independent visibility that exposes exactly the
uncertain CTP edges incident to the current vertex. Known edges receive
deterministic statuses. Set
\[
    T_{\max}=H\max_{e\in E}W(e),
\]
which also has polynomial binary encoding length. Every constructed interval
is stored by its two binary-encoded endpoints, so the full instance
description has polynomial length. A wait reveals no new status and can be
removed because all statuses and visibility sets are time-invariant and
waiting has positive cost. Likewise, a cycle returning to the same $(q,y)$
reveals nothing and can be removed. The remaining cycle-free StochSIPP
policies correspond exactly to CTP policies with the same observations and
transformed costs. Therefore, $J_0^*\leq C'$ exactly when the original CTP
optimum is at most $C$.
\end{proof}

\subsection{State and Policy Correspondence}

\begin{lemma}[Information-state sufficiency]
\label{lem:sufficiency}
Two histories ending in the same decision state $(q,t,y)$ have the same
applicable actions, action costs, and future observation probabilities. If a
proper primitive policy exists, an optimal deterministic Markov primitive
policy exists on these states.
\end{lemma}

\begin{proof}
Let $A\subseteq\mathcal{U}\setminus\operatorname{dom}(y)$ be any set of
unobserved variables, and let $x_A:A\rightarrow\{0,1\}$ be an assignment.
Mutual independence gives
\[
\begin{aligned}
 &\Pr\!\left[
   X_j=x_A(j)\ \forall j\in A
   \,\middle|\,
   X_j=y(j)\ \forall j\in\operatorname{dom}(y)
 \right]\\
 &\qquad =
 \prod_{j\in A}p_j^{x_A(j)}(1-p_j)^{1-x_A(j)}.
\end{aligned}
\]
For a fixed policy, roadmap motion and visibility are deterministic
functions of prior actions and observed values. Hence, the event of
following a particular history that ends in $(q,t,y)$ is measurable with
respect to the variables in $\operatorname{dom}(y)$; if another variable
had been sensed, it would belong to that domain. Conditioning on the full
history therefore adds no condition on the variables in $A$.
The observed factors cancel from the conditional probability, so the
result depends on $y$ only through which variables remain unobserved. In
particular, each unobserved probability remains $p_j$.

Certification, schedule feasibility, and observation visibility are
functions of $(q,t,y)$. Motion along a certified schedule is deterministic
and reveals nothing before its endpoint, so its endpoint, cost, and next
observation law also depend only on this state and the selected action. The
primitive state and action spaces are finite, and every primitive step
advances time. Backward induction from $T_{\max}$ therefore selects a
deterministic minimizing primitive action at every state of finite value.
If no proper continuation exists, the value is infinity.
\end{proof}

\begin{lemma}[Macro-action losslessness]
\label{lem:macro_lossless}
Assume the macro-action generator includes every applicable endpoint-time
pair. From any reachable state, every proper deterministic Markov primitive
policy induces a StochSIPP macro-action policy with the same goal-arrival time
for every supported complete status assignment consistent with that state.
Conversely, expanding each stored schedule induces a primitive policy with the
same pathwise arrival time.
\end{lemma}

\begin{proof}
Fix a reachable decision state and follow the primitive policy until the
first new observation event or goal arrival. Every primitive edge and
destination occupancy in this segment is certified safe. By construction,
every internal vertex--time pair has an empty newly observed set, and the
endpoint is the first pair with a nonempty set or the goal. The segment is
therefore an applicable macro-action. Its endpoint pair appears in
$\mathcal{A}(s)$ by complete generation.

If the generator stores a different certified schedule for the same
endpoint pair, the replacement has the same elapsed cost. It also has the
same next observation distribution because that distribution depends only
on the endpoint, time, and $y$. Repeating this replacement between
successive observation events preserves every branch's observation
sequence and goal-arrival time. Conversely, each stored schedule is a
finite sequence of applicable primitive edges. A primitive controller can
retain its current schedule as finite memory and execute that sequence
before consulting the next policy action. This preserves the pathwise cost
in both directions. Lemma~\ref{lem:sufficiency} supplies an optimal
deterministic Markov primitive policy whenever a proper primitive policy
exists. Therefore,
$J_{\mathrm{mac}}^*(s)=J_{\mathrm{prim}}^*(s)$ for every reachable state.
\end{proof}

\begin{proposition}[Policy correspondence]
\label{prop:policy_correspondence}
For every reachable decision state, proper deterministic Markov
macro-action policies and proper solution subgraphs have the same
conditional expected costs.
\end{proposition}

\begin{proof}
We use induction on the remaining horizon $T_{\max}-t$. At a goal state, the
policy and terminal OR node both have value zero. At a non-goal state, a
policy selects one action, represented by its AND child because generation
is complete. That child contains every positive-probability observation.
Each successor has strictly less remaining horizon, so the induction
hypothesis applies. Equation~\ref{eq:bellman-q} adds the same action cost and
the same probability-weighted successor values. Retaining the selected
action at every reachable OR node produces a solution subgraph with the
same value. Reading the retained action from a solution subgraph gives the
reverse construction by the same induction.
\end{proof}

\subsection{Bounds and Pruning}

\begin{proposition}[Admissible SIPP and action bounds]
\label{prop:admissible_bounds}
For every reachable state $s=(q,t,y)$,
\[
    h^{\mathrm{opt}}(q,t)
    \leq J^*(s)
    \leq h^{\mathrm{rob}}(q,t).
\]
If every successor satisfies
$\ell(s_a^\omega)\leq J^*(s_a^\omega)\leq u(s_a^\omega)$, then every
generated action satisfies
\[
    \underline{Q}(s,a)
    \leq Q^*(s,a)
    \leq \overline{Q}(s,a).
\]
\end{proposition}

\begin{proof}
If no proper policy exists, $J^*(s)=\infty$ and the optimistic lower bound
is immediate. Otherwise, fix an arbitrary proper policy and one of its
observation branches. Every motion on the branch is available in the
optimistic roadmap because that roadmap admits all uncertain intervals. Its
earliest goal-arrival time cannot exceed the branch's arrival time.
Subtracting $t$, taking the expectation, and then minimizing over proper
policies gives $h^{\mathrm{opt}}(q,t)\leq J^*(s)$.

If the robust query is finite, its route uses only probability-one
intervals after the certified current occupancy. The robot can follow this
route for every observation outcome; splitting it at mandatory observation
events gives a proper witness policy of cost
$h^{\mathrm{rob}}(q,t)$. Hence
$J^*(s)\leq h^{\mathrm{rob}}(q,t)$. If the query is infeasible, the upper
bound is infinity and remains valid.

For an action $a$, write
$o_\omega=O(\omega\mid\bar{s}_a)$. Multiply each successor inequality by
the nonnegative $o_\omega$ and sum over $\omega$. The outcome probabilities
sum to one. Adding the same finite $c(s,a)$ gives
\[
\begin{aligned}
 \underline{Q}(s,a)
 &=c(s,a)+\sum_\omega o_\omega\ell(s_a^\omega)\\
 &\leq Q^*(s,a)\\
 &\leq c(s,a)+\sum_\omega o_\omega u(s_a^\omega)
 =\overline{Q}(s,a),
\end{aligned}
\]
which is exactly the stated action-bound inequality.
\end{proof}

\begin{proposition}[Sound pruning and OR backups]
\label{prop:sound_pruning}
Suppose the node and action bounds are valid, and every finite upper bound
is backed by a feasible witness policy. If
$\underline{Q}(s,a)>B(s)$, then $a$ is not optimal. After any sequence of
such removals, Eqs.~\ref{eq:or-lower} and~\ref{eq:or-upper} satisfy
\[
    \ell(s)\leq J^*(s)\leq u(s).
\]
\end{proposition}

\begin{proof}
Every finite term in $B(s)$ is the value of a feasible policy: the first
term has the robust-route witness, and a finite
$\overline{Q}(s,a')$ combines the action with a feasible witness at each
successor. Therefore,
\[
    J^*(s)\leq B(s).
\]
Proposition~\ref{prop:admissible_bounds} gives
$Q^*(s,a)\geq\underline{Q}(s,a)$. For a pruned action,
\[
    Q^*(s,a)
    \geq\underline{Q}(s,a)
    >B(s)
    \geq J^*(s).
\]
It cannot attain the optimal value. The strict inequality also shows that
an optimal action is never removed, so
\[
    J^*(s)
    =
    \min_{a\in\mathcal{A}_{\mathrm{act}}(s)}Q^*(s,a).
\]
Taking the minimum of the active action lower bounds gives a lower bound on
this equality, and taking the minimum of their upper bounds gives an upper
bound. The maximum with the independent optimistic lower bound and the
minimum with the robust upper-bound witness preserve the inequalities in
Eqs.~\ref{eq:or-lower} and~\ref{eq:or-upper}. Initialization supplies the
base case, so reverse-topological backups preserve the invariant.
\end{proof}

\subsection{Safety and Finite-Roadmap Optimality}

\begin{theorem}[Primitive-step safety]
\label{thm:safety}
Assume the realized interval statuses follow the supplied support,
observations report them exactly, and execution follows the planned timing.
Every execution that starts at $q_s$ and uses only applicable StochSIPP
actions is collision-free.
\end{theorem}

\begin{proof}
Maintain two invariants: $y$ agrees with every status observed in the
current episode, and the robot safely occupies its current vertex. They hold
initially because $y_\emptyset$ contains no claims and the start occupancy
has probability one. Exact sensing preserves the first invariant.

Every interval used by an applicable schedule is certified safe. A
probability-one variable is safe under the supplied support, and an
uncertain certified variable was observed with value one. Hence every used
edge execution and destination occupancy is safe in the current episode.
The lookahead sensing condition exposes these variables before departure.
For a wait, the self-loop event covers the complete waiting period and the
destination event covers its endpoint. Induction over all primitive steps
preserves safe occupancy and proves collision freedom.
\end{proof}

\begin{theorem}[Finite-roadmap optimality]
\label{thm:optimality}
Under the assumptions of Theorem~\ref{thm:safety}, additionally assume the
interval probabilities are exact and independent, action and outcome
generation are complete, and Algorithm~\ref{alg:stochsipp} uses the stated
caches, backups, solved tests, and strict pruning rule. Then the algorithm
terminates after finitely many expansions. If a proper primitive policy
exists on the roadmap, it returns a collision-free policy with minimum
expected goal-arrival time among all such policies. Otherwise, it returns
failure.
\end{theorem}

\begin{proof}
There are at most
\[
    |V|(T_{\max}+1)3^M
\]
OR-state keys because every uncertain variable is unobserved, observed safe,
or observed blocked. Each state has finitely many endpoint pairs. Every
action strictly increases time, so the cached action--observation graph is a
finite DAG, and each OR node is expanded at most once. At an expanded node
with a nonempty complete action set, an optimal action exists because the
set is finite. Sound strict pruning never removes that action, so
$\mathcal{A}_{\mathrm{act}}(s)$ remains nonempty.

We next show that an expansion is available whenever the chance root is
unsolved. An unsolved chance or AND node has an unsolved child by its solved
test. At an expanded unsolved OR node, the selected minimum-lower-bound
action cannot be solved. If it were solved, its exact value would be no
larger than every alternative lower bound, and
Eq.~\ref{eq:or-solved} would mark the OR node solved. Following selected
actions and unsolved children in the finite DAG must therefore reach an
unexpanded OR node. A complete actionless node and an optimistic dead end
are already solved and cannot terminate this descent. Thus, every loop
iteration expands a previously unexpanded node; only finitely many
iterations are possible.

Propositions~\ref{prop:admissible_bounds} and~\ref{prop:sound_pruning} show
that every stored bound remains valid and no optimal action is removed. A
solved AND node is exact because all its positive-probability successors are
exact and Eq.~\ref{eq:bellman-q} computes their expectation. At a solved OR
node, the certifying action is exact and its value is no greater than the
lower bound, hence no greater than the true value, of every active
alternative. Pruned actions are nonoptimal. The certifying action therefore
satisfies Eq.~\ref{eq:bellman-j}. Reverse induction over the DAG shows that
all solved values, including the solved chance root, are exact.

Lemma~\ref{lem:sufficiency} permits an optimal deterministic Markov
primitive policy. Lemma~\ref{lem:macro_lossless} places an equal-cost policy
in the generated macro-action class, and
Proposition~\ref{prop:policy_correspondence} transfers the exact solution
subgraph back to that policy. If the exact root value is infinity, no proper
policy exists and the algorithm reports failure. Otherwise, the returned
policy is collision-free by Theorem~\ref{thm:safety}.
\end{proof}

\paragraph{Scope of the guarantee.}
The result is exact for the supplied finite roadmap, horizon, independent
interval variables, and sensing model. It does not imply global optimality
in the continuous configuration space. Correlated interval statuses,
incorrectly merged time ranges, sensing errors, or timing deviations violate
the model assumptions. Incorrect probability values invalidate
expected-cost optimality; safety depends on exact status support, sensing,
and timing rather than probability calibration.

\section{Extended Experimental Results}

This section reports the disaggregated values behind the main-paper
summaries. It does not add experiments or modify the recorded results.

\subsection{Additional Experimental and Implementation Details}

The reported warehouse-like environments are fifty maps
(\texttt{risky\_variant\_000}--\texttt{049}), each a 50-vertex roadmap in an
$800\times800$ workspace with five static circular obstacles. Each vertex
is connected to its six nearest neighbors, with bridge edges added when
needed for connectivity. The start and goal are near opposite corners.

Each environment contains three moving obstacles. Every obstacle
independently follows one of two candidate trajectories with equal prior
probability, producing $2^3=8$ trajectory-level realizations per map. One
StochSIPP policy is computed before execution and evaluated over 1,000
sampled realizations. The reported collision threshold is
$\texttt{ROB\_RAD}+\texttt{OBST\_RAD}=1.0$, matching the clearance used to
construct the safe intervals.

Trajectory-level sampling is not the independent interval model analyzed
in the proofs. One sampled obstacle trajectory can jointly determine
several edge/time statuses, so those statuses can be correlated even when
the obstacle-level trajectory choices are independent. The planner is
supplied interval marginals and, under its independence model, leaves the
unobserved marginals unchanged after sensing. Consequently, the
finite-roadmap expected-cost
optimality result does not transfer to these correlated evaluation worlds.
The primitive-step safety result applies only if the trajectory-to-interval
conversion yields support-correct, episode-fixed statuses, exact observations,
and execution that follows the planned timing.

Each risky-route map isolates one permanently ambiguous $p=0.5$ gate in the
reported route choice. This isolates the measured value of reacting to that
gate, but it does not establish independence for every status induced by the
trajectory generator. We therefore make no empirical claim of optimal expected
cost under the trajectory-level distribution.

The comparison includes seven fixed-path baselines, none of which revises its
path after observations. Five discount edge cost by safety probability during
search and differ only in the admission threshold: \emph{det-A*} admits every
edge, \emph{stoch-det-A*} every interval of nonzero safety probability,
\emph{0.5-A*} and \emph{0.7-A*} only intervals meeting those thresholds at
every relevant step, and \emph{1.0-A*} only intervals of safety probability
one. \emph{Most-likely} plans against the single most probable joint outcome of
all moving obstacles. \emph{Reactive} commits to the shortest route but halts
in place rather than proceeding once it observes a blocked edge ahead, so it
never collides but does not always reach the goal; its reported success rate is
therefore a goal-reaching rate rather than $1-{}$collision rate.

Two additional structured scenarios, \emph{2-route} and \emph{3-route}, place
one and two independent ambiguous gates between start and goal rather than a
randomly generated layout. They are used both as safety scenarios and as
ablation scenarios below.

\todo{Report the experiment hardware, software versions, random-seed
propagation, exact baseline cost formula, and per-cell scaling timeout
duration before submission.}
\todo{Document how sampled trajectories produce the interval partition,
episode-fixed statuses, visibility sets, and observations used at execution.}

\subsection{Two-Dimensional Scenarios}

The first set shows five representative randomly generated roadmaps from the
two-dimensional evaluation, illustrating the range of obstacle layouts and
route connectivity used in the experiments.

\begin{figure*}[p]
  \centering
  \begin{minipage}[t]{0.37\textwidth}
    \centering
    \includegraphics[width=\linewidth]{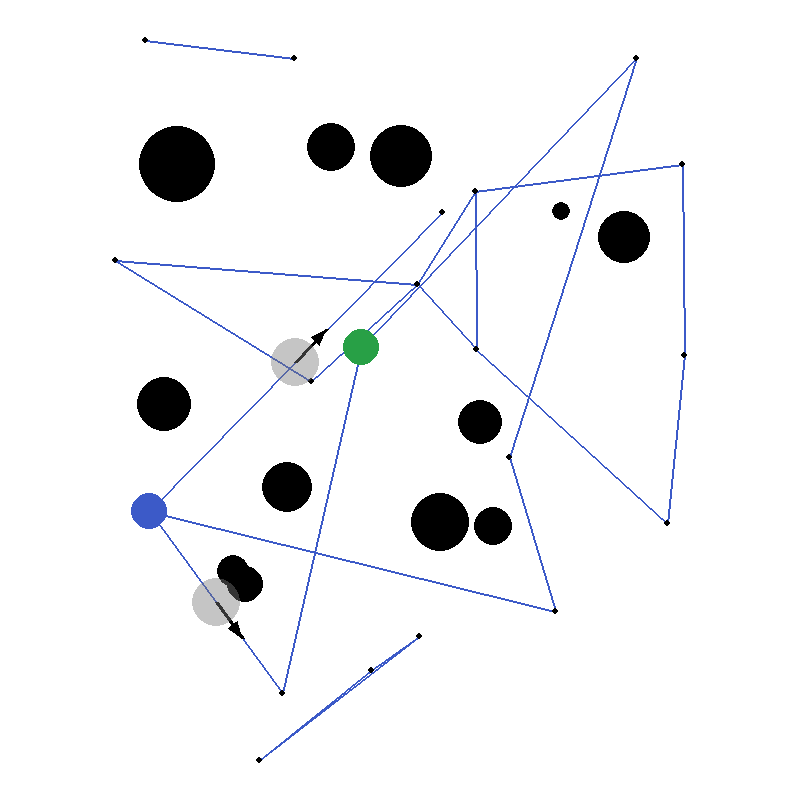}
    \scriptsize (a) Scenario 1
  \end{minipage}\hfill
  \begin{minipage}[t]{0.37\textwidth}
    \centering
    \includegraphics[width=\linewidth]{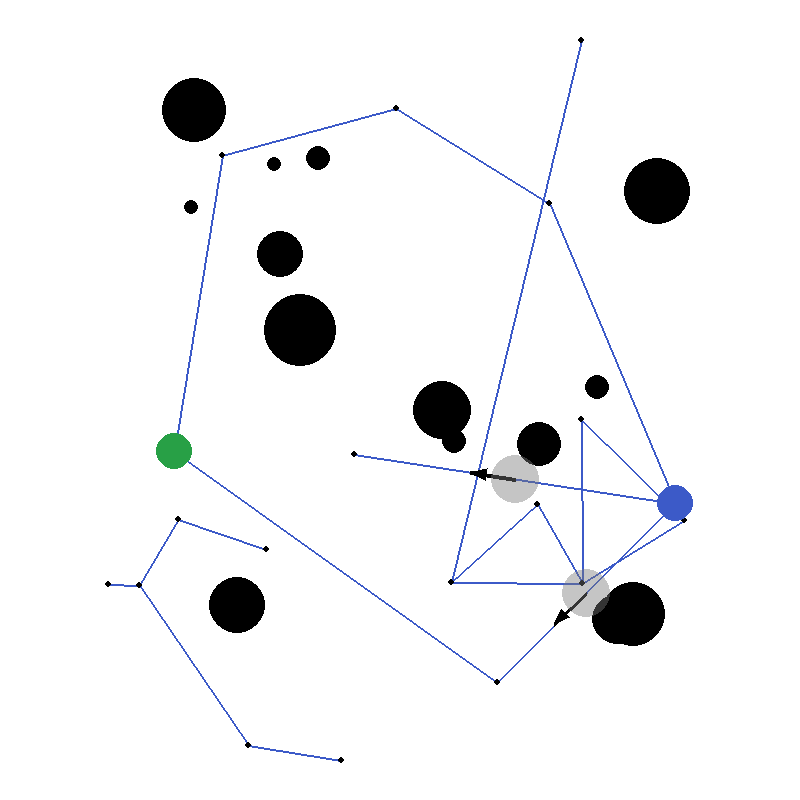}
    \scriptsize (b) Scenario 2
  \end{minipage}

  \smallskip
  \begin{minipage}[t]{0.37\textwidth}
    \centering
    \includegraphics[width=\linewidth]{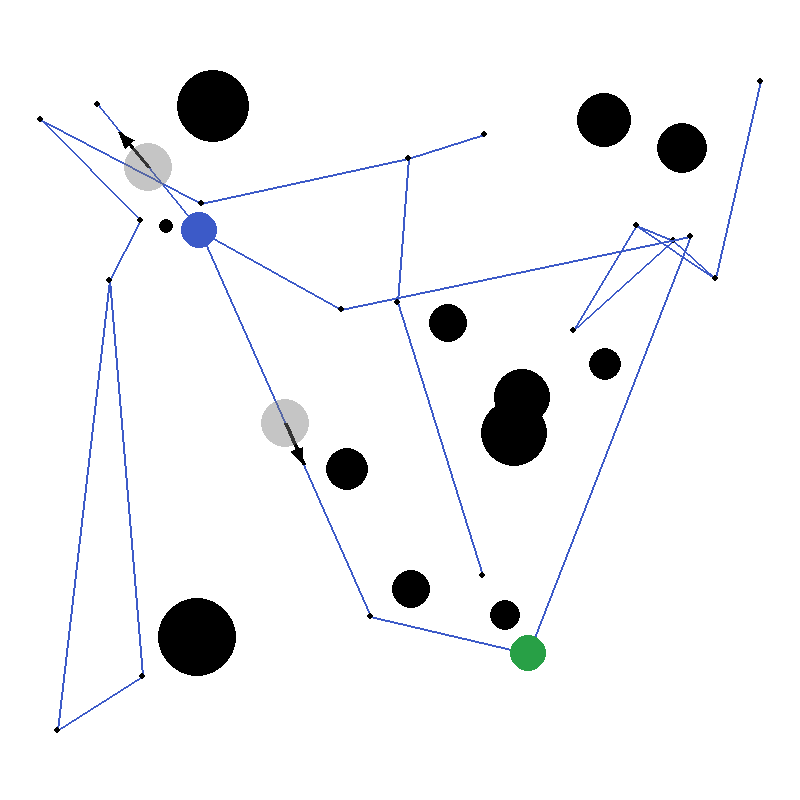}
    \scriptsize (c) Scenario 3
  \end{minipage}\hfill
  \begin{minipage}[t]{0.37\textwidth}
    \centering
    \includegraphics[width=\linewidth]{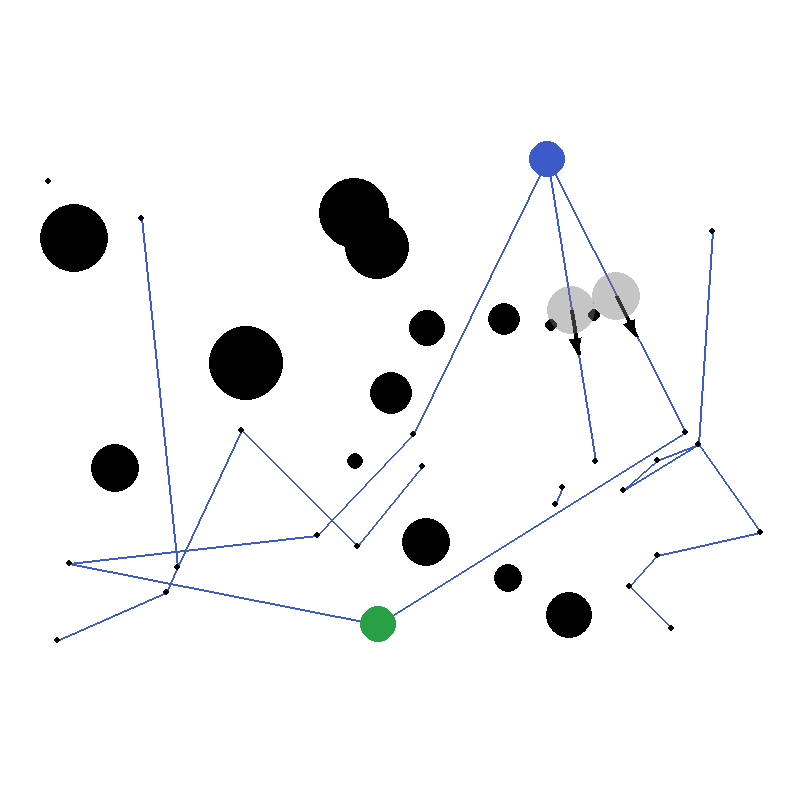}
    \scriptsize (d) Scenario 4
  \end{minipage}

  \smallskip
  \begin{minipage}[t]{0.37\textwidth}
    \centering
    \includegraphics[width=\linewidth]{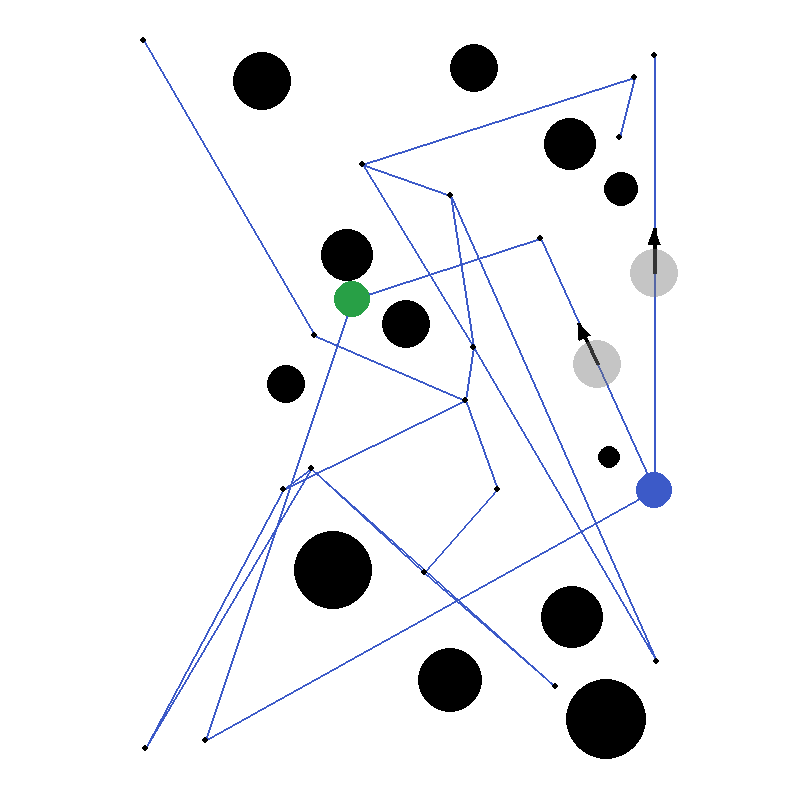}
    \scriptsize (e) Scenario 5
  \end{minipage}
  \caption{Five representative scenarios used in the two-dimensional
  experiments. Across all two-dimensional scenarios, the blue circle marks
  the start, the green circle marks the goal, and the gray circles mark the
  moving obstacles' potential starting locations. Arrows indicate their
  potential directions of movement.}
  \label{fig:supp-scenarios}
\end{figure*}
\FloatBarrier

The second set shows two controlled gated scenarios that concentrate
uncertainty at route choices, highlighting the planner's ability to defer a
route commitment until the relevant gate is observed.

\begin{figure}[H]
  \centering
  \begin{minipage}[t]{0.82\columnwidth}
    \centering
    \includegraphics[width=\linewidth]{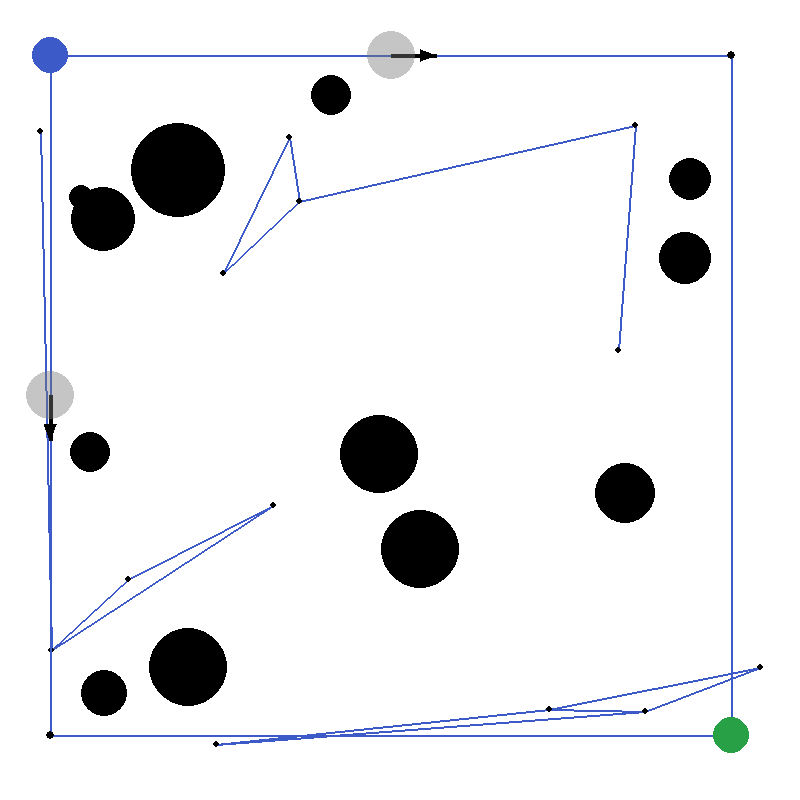}
    \scriptsize (a) Gated Scenario 1
  \end{minipage}

  \medskip
  \begin{minipage}[t]{0.82\columnwidth}
    \centering
    \includegraphics[width=\linewidth]{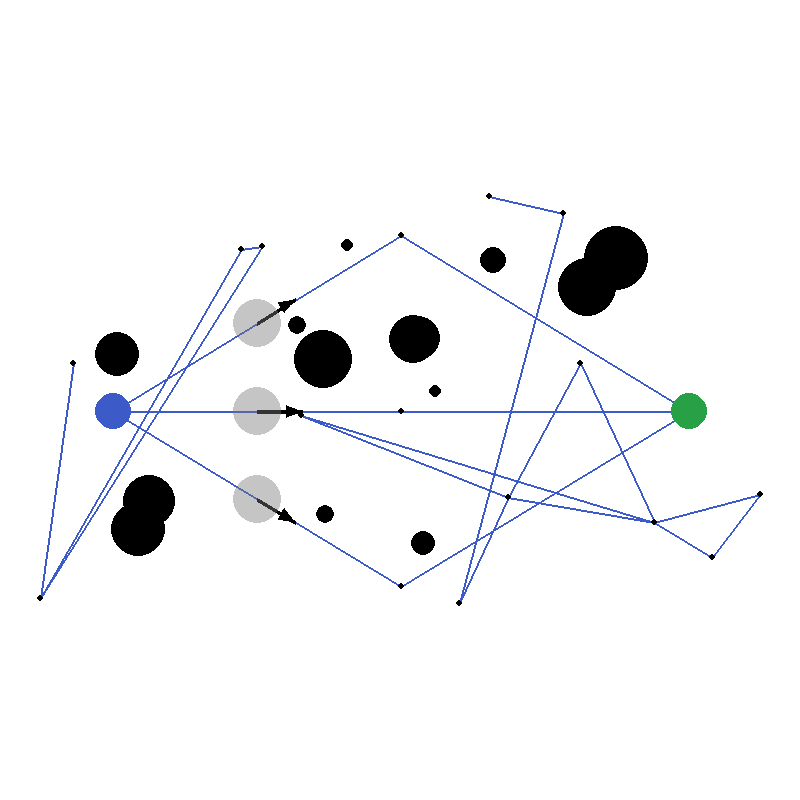}
    \scriptsize (b) Gated Scenario 2
  \end{minipage}
  \caption{The two gated scenarios used in the two-dimensional experiments,
  following the same visual conventions as the representative scenarios.}
  \label{fig:supp-gated-scenarios}
\end{figure}

\subsection{Controlled Single-Gate Results}

Table~\ref{tab:supp-results} reports the values averaged over the fifty maps,
and Table~\ref{tab:supp-paired} reports the same comparison as per-map paired
differences from StochSIPP. On every one of the fifty maps the seven
baselines split into exactly two clusters that each commit to an identical
route, which is why the paired differences within a cluster are identical and
their intervals have zero width where the underlying values are constant.

\begin{table}[t]
\centering
\small
\setlength{\tabcolsep}{4pt}
\begin{tabular}{lcccc}
\toprule
Method & Mean Arr. & All-Pass & Success & Solve (s) \\
\midrule
StochSIPP & 11.6 & 7.9 & 100.0\% & 0.033 \\
1.0-A* & 14.1 & 14.1 & 100.0\% & 0.000 \\
0.7-A* & 14.1 & 14.1 & 100.0\% & 0.000 \\
most-likely & 14.1 & 14.1 & 100.0\% & 0.000 \\
det-A* & 6.9 & 6.9 & 49.7\% & 0.000 \\
stoch-det-A* & 6.9 & 6.9 & 49.7\% & 0.000 \\
0.5-A* & 6.9 & 6.9 & 49.7\% & 0.000 \\
reactive & 6.9 & 6.9 & 49.7\% & 0.000 \\
\bottomrule
\end{tabular}
\caption{Mean arrival time, mean arrival restricted to the trials in which
every method reached the goal without colliding, and success rate
($1-{}$collision rate, or goal-reaching rate for \emph{reactive}), averaged
over 1,000 trials per map across the fifty risky maps.}
\label{tab:supp-results}
\end{table}

\begin{table}[t]
\centering
\small
\setlength{\tabcolsep}{4pt}
\begin{tabular}{lccc}
\toprule
Cluster & $\Delta$Mean Arr. & $\Delta$All-Pass & $\Delta$Success \\
\midrule
Cautious & $-2.6\pm0.4$ & $-6.2\pm0.8$ & $0.0\pm0.0$ \\
Risky & $+4.6\pm0.4$ & $+1.0\pm0.0$ & $+50.3\pm0.0$ \\
\bottomrule
\end{tabular}
\caption{Paired differences $\Delta=\text{StochSIPP}-\text{baseline}$ across the
fifty maps, mean $\pm$ half-width of the $95\%$ confidence interval.
The cautious cluster is \emph{0.7-A*}, \emph{1.0-A*}, and \emph{most-likely};
the risky cluster is \emph{det-A*}, \emph{stoch-det-A*}, \emph{0.5-A*}, and
\emph{reactive}. Members of a cluster produce identical differences.
$\Delta$Success is in percentage points. $\Delta$Solve is
$0.033\pm0.004$~s against every baseline.}
\label{tab:supp-paired}
\end{table}

Relative to the cautious cluster, the only other methods that succeed in every
trial, StochSIPP's unrestricted mean arrival is $7\%$ to $31\%$ lower across the
fifty maps ($17\%$ on average). Relative to the risky cluster, the recorded
arrival times of those baselines are $30\%$ to $161\%$ lower than StochSIPP's
($72\%$ on average), but that speed is bought by gambling on the ambiguous gate:
\emph{det-A*}, \emph{stoch-det-A*}, and \emph{0.5-A*} collide on $50.3\%$ of
trials, and \emph{reactive} instead fails to reach the goal on the same fraction.

The unrestricted Mean Arr.\ statistic is the relevant cost summary for the
methods that succeed in every trial. For a baseline that collides, its finite
recorded duration is not the expected goal-arrival time of a proper policy and
should not be compared as though failure reached the goal. The All-Pass column
is also conditional on favorable worlds: selecting only worlds in which the
risk-taking baselines succeed selects worlds in which the short gate is open.
This is why StochSIPP's own All-Pass value ($7.9$) is well below its unrestricted
mean ($11.6$) even though its policy does not change.

\todo{Paste the per-map rows for all fifty maps from the results file; only the
fifty-map aggregates and paired differences are recorded in the manuscript
source.}

\subsection{Gated Scenarios}

Table~\ref{tab:supp-gates} reports the two structured gate scenarios. With only
ambiguous gates standing between start and goal, no route clears the $0.7$ or
$1.0$ admission threshold, so \emph{0.7-A*} and \emph{1.0-A*} return no plan at
all rather than a risky one. StochSIPP reaches the goal on every trial of both
scenarios because it disambiguates each gate rather than ruling it out in
advance. \emph{Most-likely} changes cluster relative to the single-gate maps:
independently taking each obstacle's more probable hypothesis no longer
guarantees a jointly safe route once two gates are involved, and it records a
success rate slightly above the other risky-cluster baselines
($50.3\%$ versus $49.7\%$ on \emph{2-route}, $66.4\%$ versus $65.7\%$ on
\emph{3-route}).

\begin{table}[t]
\centering
\small
\setlength{\tabcolsep}{4pt}
\begin{tabular}{lcccc}
\toprule
& \multicolumn{2}{c}{2-route} & \multicolumn{2}{c}{3-route} \\
Method & Succ. & Arr. & Succ. & Arr. \\
\midrule
StochSIPP & 100.0\% & 7.0 & 100.0\% & 7.0 \\
most-likely & 50.3\% & 6.0 & 66.4\% & 6.0 \\
det-A* & 49.7\% & 6.0 & 65.7\% & 6.0 \\
stoch-det-A* & 49.7\% & 6.0 & 65.7\% & 6.0 \\
0.5-A* & 49.7\% & 6.0 & 65.7\% & 6.0 \\
reactive & 49.7\% & 6.0 & 65.7\% & 6.0 \\
0.7-A* & 0.0\% & -- & 0.0\% & -- \\
1.0-A* & 0.0\% & -- & 0.0\% & -- \\
\bottomrule
\end{tabular}
\caption{Success rate and mean arrival on the \emph{2-route} (one ambiguous
gate) and \emph{3-route} (two ambiguous gates) scenarios, 1,000 trials each.
\emph{0.7-A*} and \emph{1.0-A*} admit no route to the goal on either scenario
and therefore return no plan.}
\label{tab:supp-gates}
\end{table}

\subsection{Search-Aid Ablation}

The ablation compares the combined search configurations used in the
implementation. \emph{Full} includes the SIPP-derived bounds, informed
probability-aware expected-value ordering estimate, and cached arrival-time
functions. \emph{Full-LRTA} adds the
implementation's per-branch refinement. \emph{No-Heuristic} removes the
bounds, pruning, and informed ordering. Because several components change
together, this experiment measures the configurations as packages; it does
not isolate the causal contribution of each component.

\begin{table}[h]
\centering
\small
\begin{tabular}{llccc}
\toprule
Seed & Variant & Expansions & Nodes & Solve (s) \\
\midrule
\multirow{3}{*}{3} & full & 3 & 26 & 0.04 \\
 & no-heur & 11487 & 40584 & 4.43 \\
 & full-lrta & 3 & 26 & 0.57 \\
\midrule
\multirow{3}{*}{7} & full & 7 & 76 & 0.12 \\
 & no-heur & 24927 & 90030 & 10.20 \\
 & full-lrta & 7 & 76 & 1.71 \\
\midrule
\multirow{3}{*}{8} & full & 3 & 54 & 0.06 \\
 & no-heur & 3695 & 13390 & 1.42 \\
 & full-lrta & 3 & 54 & 0.97 \\
\midrule
\multirow{3}{*}{10} & full & 3 & 56 & 0.06 \\
 & no-heur & 5039 & 18134 & 1.99 \\
 & full-lrta & 3 & 56 & 0.92 \\
\midrule
\multirow{3}{*}{22} & full & 3 & 46 & 0.08 \\
 & no-heur & 7175 & 25408 & 3.19 \\
 & full-lrta & 3 & 46 & 1.05 \\
\bottomrule
\end{tabular}
\caption{Search effort on the five random seeds under \emph{full},
\emph{full-lrta}, and \emph{no-heur}, one deterministic solve per seed and
variant.}
\label{tab:supp-ablations}
\end{table}

All three variants returned the same recorded root value on each random seed:
$J_0^*=7.0$ for seeds 3 and 7, and $J_0^*=6.0$ for seeds 8, 10, and 22, and
they agree with each other on \emph{2-route} and \emph{3-route} as well.
\emph{Full} expanded 3--7 OR nodes on the five random seeds, whereas
\emph{No-Heuristic} expanded 3,695--24,927, so on these instances
\emph{No-Heuristic} expanded 1,232--3,829 times as many OR nodes.
\emph{Full-LRTA} expanded the same nodes as \emph{Full}; using the reported
rounded times, it took 13.1--16.2 times as long.

The two structured gate scenarios are much larger for the search:
\emph{Full} expands 3,837 OR nodes on \emph{2-route} and 13,585 on
\emph{3-route}, and its advantage over \emph{No-Heuristic} narrows to
$16\times$ and to under $4\times$, respectively. Pooling all seven scenarios
gives the paired differences in Table~\ref{tab:supp-ablation-paired}. The
intervals are wide because the seven scenarios span three orders of magnitude
in raw effort, yet every interval still excludes zero.

\begin{table}[h]
\centering
\small
\setlength{\tabcolsep}{3pt}
\begin{tabular}{lccc}
\toprule
Variant & $\Delta$Expands & $\Delta$Nodes & $\Delta$Solve (s) \\
\midrule
no-heur & $-21210\pm18712$ & $-62320\pm47816$ & $-4.15\pm2.71$ \\
full-lrta & $0\pm0$ & $0\pm0$ & $-1.12\pm0.44$ \\
\bottomrule
\end{tabular}
\caption{Paired differences $\Delta=\text{\emph{full}}-\text{variant}$ across
the seven ablation scenarios (seeds 3, 7, 8, 10, 22 plus \emph{2-route} and
\emph{3-route}), mean $\pm$ half-width of the $95\%$ confidence interval
($t_{0.975,6}=2.45$). Negative values favor \emph{full}.}
\label{tab:supp-ablation-paired}
\end{table}

\todo{Paste the exact \emph{no-heur} and \emph{full-lrta} rows for
\emph{2-route} and \emph{3-route}; only \emph{full}'s expansion counts and the
pooled differences are recorded in the manuscript source.}

\subsection{Simultaneous-Ambiguity Scaling}

The scaling study uses a star scenario with a safe start-to-goal edge and $k$
independent moving obstacles, each adding one permanently ambiguous dead-end
spur off the start. Thus $k$ ambiguous statuses are presented together while the
roadmap has $k+2$ vertices, so any change in solve time comes from the number of
simultaneously ambiguous statuses rather than from map size.
Table~\ref{tab:supp-scaling} preserves every reported cell.

\begin{table}[h]
\centering
\small
\setlength{\tabcolsep}{4pt}
\begin{tabular}{rrrrrr}
\toprule
Ambig.\ statuses & Build (s) & Solve (s) & Expands & Nodes & A* (s) \\
\midrule
2 & 0.003 & 0.002 & 29 & 116 & 0.000 \\
3 & 0.005 & 0.016 & 175 & 796 & 0.000 \\
4 & 0.006 & 0.179 & 1,401 & 7,140 & 0.000 \\
5 & 0.010 & 2.459 & 14,011 & 79,084 & 0.000 \\
6 & 0.012 & 39.639 & 168,133 & 1,041,172 & 0.000 \\
\bottomrule
\end{tabular}
\caption{Recorded build time, StochSIPP solve time, search effort, and
A* planning time as the number of ambiguous statuses presented together grows.
Settings beyond six did not resolve within the sweep's timeout.}
\todo{Record which settings beyond six were attempted and the wall-clock
timeout duration.}
\label{tab:supp-scaling}
\end{table}

Across the five completed points, build time increased only from $0.003$~s to
$0.012$~s, while solve time increased from $0.002$~s to $39.639$~s. Each
additional ambiguous status multiplies solve time by a larger factor than the
one before: roughly $8\times$ from 2 to 3 statuses, $11\times$ from 3 to 4,
$14\times$ from 4 to 5, and $16\times$ from 5 to 6. Expansions and constructed
nodes compound in the same way, and the A* baselines, which never branch on
outcomes, plan in well under a millisecond at every setting. Five completed
settings are insufficient to identify an asymptotic growth law or to attribute
the observed growth exclusively to one implementation mechanism.

\subsection{Gazebo Illustration}

\begin{figure*}[h]
  \centering
  \begin{minipage}[t]{0.49\textwidth}
    \centering
    \includegraphics[width=\linewidth]{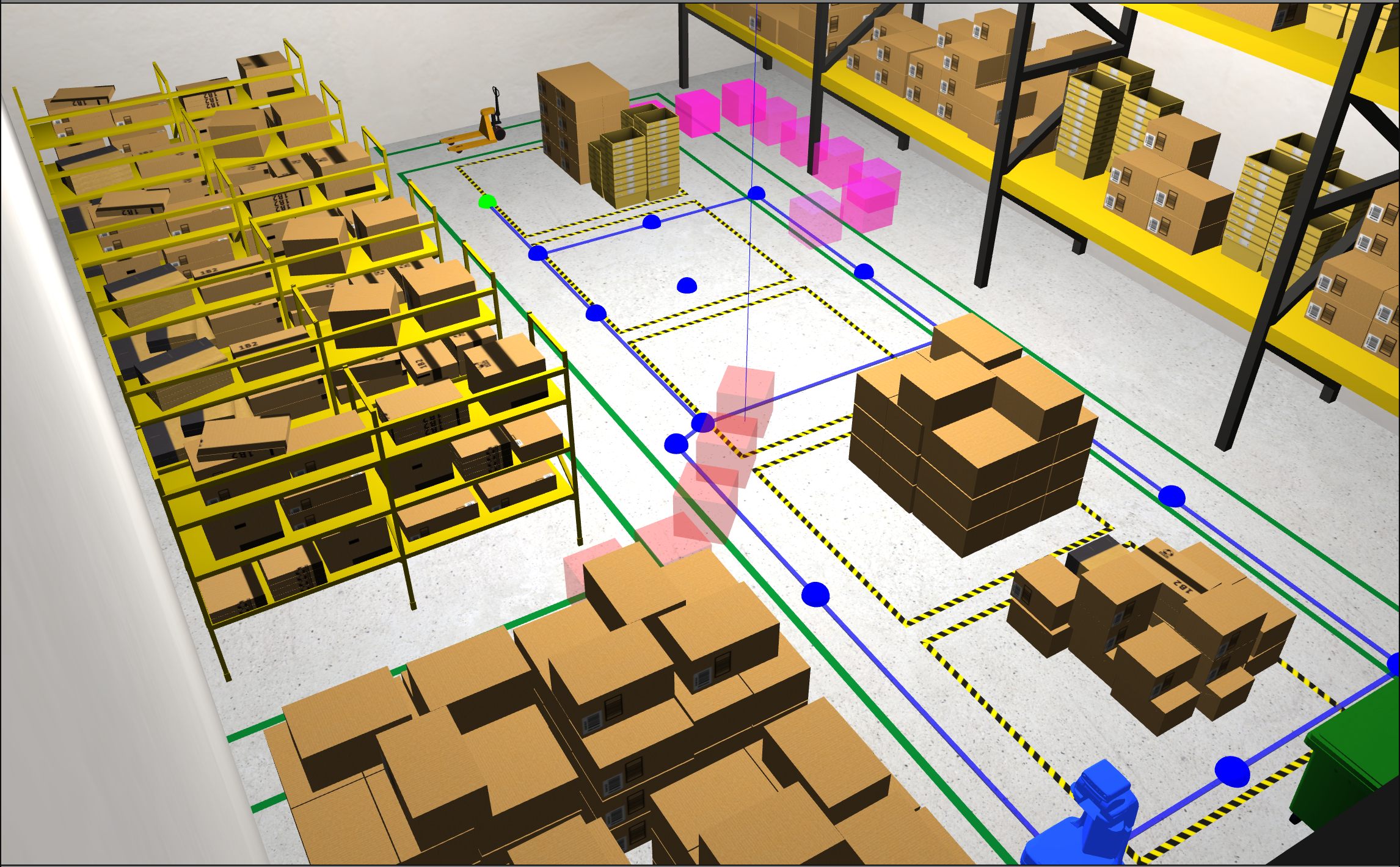}
    \scriptsize (a) Motion pattern 1
  \end{minipage}\hfill
  \begin{minipage}[t]{0.49\textwidth}
    \centering
    \includegraphics[width=\linewidth]{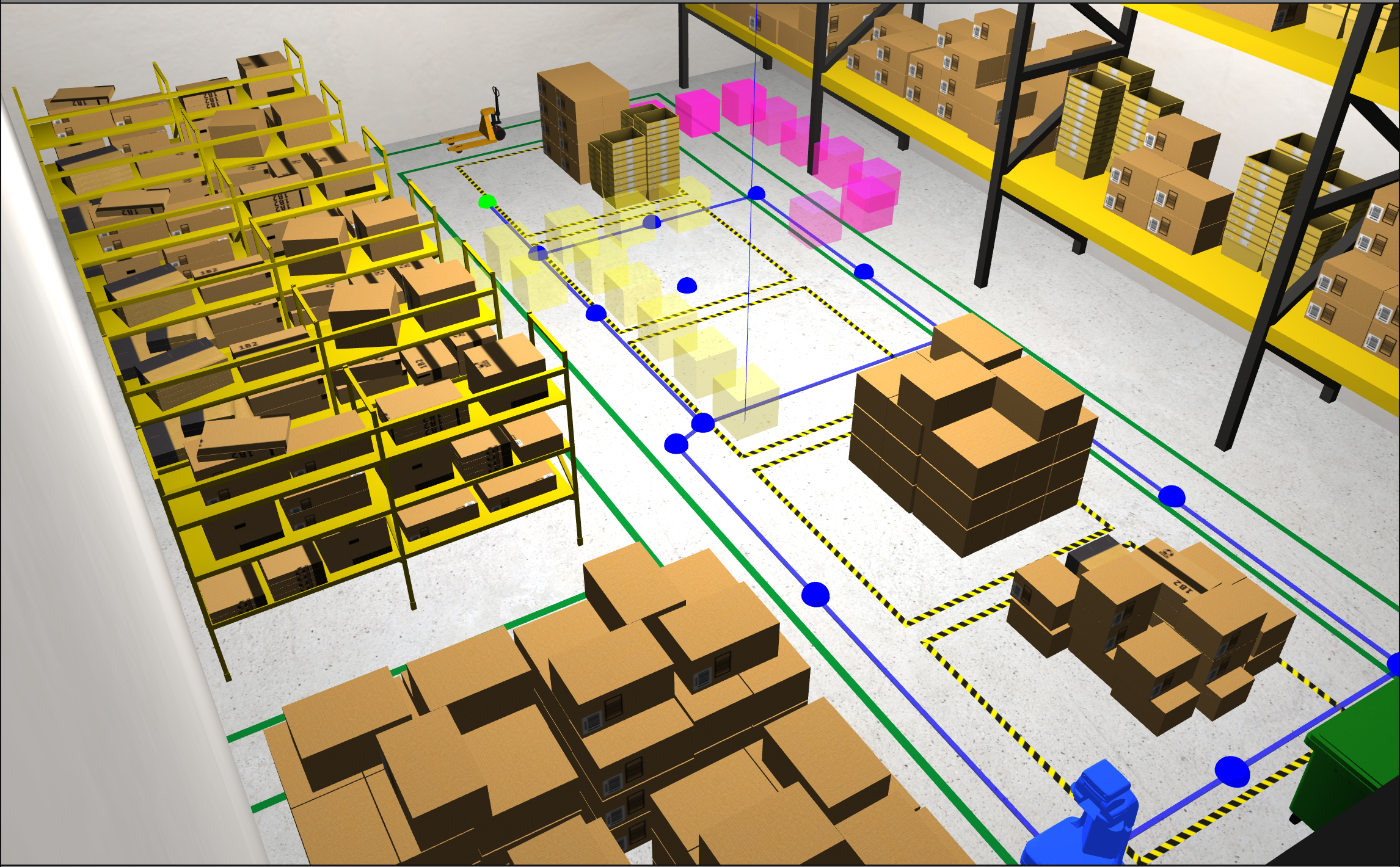}
    \scriptsize (b) Motion pattern 2
  \end{minipage}
  \caption{The two obstacle-motion patterns used for the warehouse results.}
  \label{fig:supp-warehouse-motion-patterns}
\end{figure*}

The Gazebo rendering of a representative warehouse policy execution appears as
Figure~\ref{fig:supp-warehouse-motion-patterns}. It is a qualitative illustration of the
policy-execution setup and provides no evidence for collision rate,
optimality, or runtime.

\begin{table}[t]
  \centering
  \caption{Planner comparison on the \texttt{warehouse} scenario over
    1000 trials (2 distinct worlds, collision radius $1.2$\,m). Both routes
    from the start to the goal cross a permanently ambiguous gate, so the
    certainty-demanding baselines cannot plan at all, while the permissive
    ones must commit blindly and collide roughly half the time.
    CAO$^{*}$ is the only method that reaches the goal in every trial
    without colliding.}
  \label{tab:warehouse}
  \small
  \setlength{\tabcolsep}{4pt}
  \begin{tabular}{lrrrrl}
    \toprule
    Method
      & \shortstack[r]{Mean\\arrival}
      & \shortstack[r]{Mean arrival\\(all pass)}
      & \shortstack[r]{Collision\\rate}
      & \shortstack[r]{Wall clock\\(s)}
      & Status \\
    \midrule
    \textbf{CAO$^{*}$} & \textbf{19.0} & --- & \textbf{0.0\%} & 0.090     & \textbf{ok} \\
    reactive           & 14.0          & --- & 0.0\%          & $<\!0.001$ & no goal in 503/1000 \\
    det-A$^{*}$        & 14.0          & --- & 50.3\%         & $<\!0.001$ & ok \\
    stoch-det-A$^{*}$  & 14.0          & --- & 50.3\%         & $<\!0.001$ & ok \\
    0.5-A$^{*}$        & 22.0          & --- & 49.7\%         & $<\!0.001$ & ok \\
    0.7-A$^{*}$        & ---           & --- & ---            & $<\!0.001$ & no path \\
    1.0-A$^{*}$        & ---           & --- & ---            & $<\!0.001$ & no path \\
    most-likely        & 22.0          & --- & 49.7\%         & $<\!0.001$ & ok \\
    \bottomrule
  \end{tabular}
\end{table}


\end{document}